\documentclass{article}

\usepackage{microtype}
\usepackage{graphicx}
\usepackage{subfigure}
\usepackage{booktabs} %

\usepackage{hyperref}

 \usepackage[accepted]{icml2023}

\usepackage{amsmath}
\usepackage{amssymb}
\usepackage{mathtools}
\usepackage{amsthm}

\usepackage[capitalize,noabbrev]{cleveref}

\theoremstyle{plain}
\newtheorem{theorem}{Theorem}[section]

\newtheorem{lemma}[theorem]{Lemma}
\newtheorem{corollary}[theorem]{Corollary}
\theoremstyle{definition}
\newtheorem{definition}[theorem]{Definition}
\newtheorem{assumption}[theorem]{Assumption}
\theoremstyle{remark}
\newtheorem{remark}[theorem]{Remark}

\usepackage[textsize=tiny]{todonotes}

\icmltitlerunning{Distributional Offline Policy Evaluation with Predictive Error Guarantees}

\usepackage{nicefrac}
\usepackage{amsfonts}
\usepackage{dsfont}
\usepackage{setspace}

\newcommand{\alglinelabel}{
  \addtocounter{ALC@line}{-1}
  \refstepcounter{ALC@line}
  \label
}

\begin{document}

\twocolumn[
\icmltitle{Distributional Offline Policy Evaluation with Predictive Error Guarantees}

\icmlsetsymbol{equal}{*}

\begin{icmlauthorlist}
\icmlauthor{Runzhe Wu}{yyy}
\icmlauthor{Masatoshi Uehara}{yyy}
\icmlauthor{Wen Sun}{yyy}
\end{icmlauthorlist}

\icmlaffiliation{yyy}{Department of Computer Science, Cornell University, Ithaca, NY, USA}

\icmlcorrespondingauthor{Runzhe Wu}{rw646@cornell.edu}

\icmlkeywords{Machine Learning, ICML}

\vskip 0.3in
]

\printAffiliationsAndNotice{}  %

\begin{abstract}
We study the problem of estimating the distribution of the return of a policy using an offline dataset that is not generated from the policy, i.e., distributional offline policy evaluation (OPE). We propose an algorithm called Fitted Likelihood Estimation (FLE), which conducts a sequence of Maximum Likelihood Estimation (MLE) and has the flexibility of integrating any state-of-the-art probabilistic generative models as long as it can be trained via MLE. FLE can be used for both finite-horizon and infinite-horizon discounted settings where rewards can be multi-dimensional vectors. Our theoretical results show that for both finite-horizon and infinite-horizon discounted settings, FLE can learn distributions that are close to the ground truth under total variation distance and Wasserstein distance, respectively. Our theoretical results hold under the conditions that the offline data covers the test policy's traces and that the supervised learning MLE procedures succeed. Experimentally, we demonstrate the performance of FLE with two generative models, Gaussian mixture models and diffusion models. For the multi-dimensional reward setting, FLE with diffusion models is capable of estimating the complicated distribution of the return of a test policy. 
\end{abstract}

\def\*#1{\bm{#1}} 
\def\+#1{\mathcal{#1}}
\def\=#1{\mathbb{#1}}
\def\^#1{\hat{#1}} 
\def\-#1{\bar{#1}}
\def\~#1{\tilde{#1}}
\def\w~#1{\widetilde{#1}}

\def\indic{\mathds{1}}
\def\d{\mathop{}\!\mathrm{d}}
\def\E{\mathop{\mathbb{E}}}
\def\KL{D_{\rm{KL}}}
\def\argmax{\mathop{\arg\max}}
\def\argmin{\mathop{\arg\min}}
\def\given{\,|\,}

\newcommand\edit[1]{{#1}}
\newcommand\numberthis{\addtocounter{equation}{1}\tag{\theequation}}

\section{Introduction}

Traditional Reinforcement Learning (RL) focuses on studying the expected behaviors of a learning agent. However, modeling the expected behavior is not enough for many interesting applications. For instance, when estimating the value of a  new medical treatment, instead of just predicting its expected value, we may be interested in estimating the variance of the value as well. For a self-driving car whose goal is to reach a destination as soon as possible, in addition to predicting the expected traveling time, we may be interested in estimating the tails of the distribution of traveling time so that customers can prepare for worst-case situations. Other risk-sensitive applications in finance and control often require one to model beyond the expectation as well. 

In this work, we study how to estimate the distribution of the return of a policy in Markov Decision Processes (MDPs) using only an offline dataset that is not necessarily generated from the test policy (i.e., distributional offline policy evaluation). Estimating distributions of returns has been studied in the setting called distributional RL \citep{bellemare2017distributional}, where most existing works focus on solving the regular RL problem, i.e., finding a policy that maximizes the expected return by treating the task of predicting additional information beyond the mean as an auxiliary task. Empirically, it is believed that this auxiliary task helps representation learning which in turn leads to better empirical performance. Instead of focusing on this auxiliary loss perspective, we aim to design distributional OPE algorithms, which can accurately estimate the distribution of returns with provable guarantees. We are also interested in the setting where the one-step reward could be \emph{multi-dimensional} (i.e., multi-objective RL), and the state/action spaces could be large or even continuous. This requires us to design new algorithms that can leverage rich function approximation (e.g., state-of-art probabilistic generative models). 

Our algorithm, \emph{Fitted Likelihood Estimation} (FLE), is inspired by the classic OPE algorithm Fitted Q Evaluation (FQE) \citep{munos2008finite}. Given a test policy and an offline dataset, FLE iteratively calls a supervised learning oracle --- Maximum Likelihood Estimation (MLE) in this case,  to fit a conditional distribution to approximate a target distribution constructed using the distribution learned from the previous iteration. At the end of the training procedure, it outputs an estimator which approximates the true distribution of the return of the test policy. Our algorithm is simple: like FQE, it decomposes the distributional OPE problem into a sequence of supervised learning problems (in this case, MLE). Thus it has great flexibility to leverage any state-of-art probabilistic generative models as long as it can be trained via MLE.  Such flexibility is important,  especially when we have large state/action spaces, and reward vectors coming from complicated high-dimensional distributions. FLE naturally works for both finite-horizon setting and infinite-horizon discounted setting. \looseness=-1

Theoretically, we prove that our algorithm, FLE, can learn an accurate estimator of the return distribution for both finite-horizon MDPs and infinite-horizon discounted MDPs, under the assumptions that  (1) \emph{MLE can achieve good in-distribution generalization bounds (i.e., supervised learning succeeds)}, and (2) \emph{the offline state-action distribution covers the test policy's state-action distribution}. The first condition is well studied in statistical learning theory, and in practice, the state-of-the-art probabilistic generative models trained via MLE (e.g., FLOW models \citep{dinh2014nice} and Diffusion models \citep{sohl2015deep}) indeed also exhibit amazing generalization ability.  The second condition is necessary for offline RL and is widely used in the regular offline RL literature (e.g., \citet{munos2008finite}). In other words, our analysis is modular: it simply transfers the supervised learning MLE in-distribution generalization bounds to a bound of distributional OPE. The accuracy of the estimator computed by  FLE is measured under total variation distance and $p$-Wasserstein distance, for finite-horizon setting and infinite-horizon discounted setting, respectively.   To complete the picture, we further provide concrete examples showing that MLE can provably have small in-distribution generalization errors. To the best of our knowledge, this is the first PAC (Probably Approximately Correct) learning algorithm for distributional OPE with general function approximation. \looseness=-1

Finally, we demonstrate our approach on a rich observation combination lock MDP where it has a latent structure with the observations being high-dimensional and continuous \citep{misra2020kinematic,agarwal2020pc,zhang2022efficient}. We consider the setting where the reward comes from complicated multi-dimensional continuous distributions (thus existing algorithms such as quantile-regression TD \citep{dabney2018distributional} do not directly apply here). We demonstrate the flexibility of our approach by using two generative models in FLE: the classic Gaussian mixture model and state-of-the-art diffusion model \citep{ho2020denoising}. \looseness=-1

\subsection{Related Works}

\textbf{Distributional RL.} 
Quantile regression TD \citep{dabney2018distributional} is one of the common approaches for distributional OPE. A very recent work~\citep{rowland2023analysis} demonstrates that quantile regression TD can converge to the TD fixed point solution of which the existence is proved under an $\ell_{\infty}$-style norm (i.e., $\sup$ over all states).  \citet{rowland2023analysis} do not consider the sample complexity of OPE and the impact of learning from off-policy samples, and their convergence analysis is asymptotic. Also, quantile regression TD only works for scalar rewards. 
Another popular approach is categorical TD \citep{bellemare2017distributional}, where one explicitly discretizes the return space. However, for high-dimensional rewards,  explicitly discretizing the return space evenly can suffer the curse of dimensionality and fail to capture some low-dimensional structures in the data distribution. Moreover, there is no convergence or sample complexity analysis of the categorical algorithm for OPE. 
Another direction in distributional RL concentrates on estimating cumulative distribution functions (CDFs) instead of densities \citep{zhang2022functional,prashanth2022wasserstein}. In addition, there are also methods based on generative models that aim to effectively represent continuous return distributions \citep{freirich2019distributional,doan2018gan,li2021bayesian}. We discuss some closely related works below.

\citet{ma2021conservative} studied distributional offline policy optimization. They focused on tabular MDPs with scalar rewards, and their algorithm can learn a pessimistic estimate of the true inverse CDF of the return. \citet{keramati2020being} also uses the distributional RL framework to optimistically estimate the CVaR value of a policy's return. Their analysis also only applies to tabular MDPs with scalar rewards.  
In contrast, we focus on distributional OPE with general function approximation beyond tabular or linear formats and MDPs with multi-dimensional rewards. 

\citet{zhang2021distributional} also consider learning from vector-valued rewards. They propose a practical algorithm that minimizes the Maximum Mean Discrepancy (MMD) without a sample complexity analysis. In contrast, we use MLE to minimize total variation distance, and our error bound is based on total variation distance. Note that a small total variation distance implies a small MMD but not vice versa, which implies that our results are stronger.

\citet{huang2021off,huang2022off} explore return distribution estimation for contextual bandits and MDPs using off-policy data. They focus on learning CDFs with an estimator that leverages importance sampling and learns the transition and reward of the underlying MDP to reduce variance while maintaining unbiasedness. However, their estimator can incur exponential error in the worst case due to importance sampling. Moreover, they measure estimation error using the $\ell_\infty$ norm on CDFs, which is upper bounded by total variation distance but not the other way around. They further showed how to estimate a range of risk functionals via the estimated distribution. Notably, our method is also applicable to risk assessment, as shown in \cref{rmk:cvar}.

\textbf{Offline policy evaluation.}  %
Fitted Q evaluation (FQE) \citep{munos2008finite,ernst2005tree} is one of the most classic OPE algorithms. Many alternative approaches have been recently proposed, such as minimax algorithms \citep{yang2020off,feng2019kernel,uehara2020minimax}. Somewhat surprisingly, algorithms based on FQE are often robust and achieve stronger empirical performance in various benchmark tasks \citep{fu2021benchmarks,chang2022learning}. Our proposed algorithm can be understood as a direct generalization of FQE to the distributional setting. %
Note sequential importance sampling approaches \citep{jiang2016doubly,precup2000eligibility} in regular RL have been applied to estimate distributions \citep{chandak2021universal}. However, these methods suffer from the curse of the horizon, i.e., the variance necessarily grows exponentially in the horizon. 
\section{Preliminaries}

In this section, we introduce the setup of the Markov decision process and the offline policy evaluation.

\textbf{Notations.}
We define $\Delta(\+S)$ as the set of all distributions over a set $\+S$. For any $a,b\in\=R$, we denote $[a,b]=\{x\in\=R:a\leq x\leq b\}$. For any integer $N$, we denote $[N]$ as the set of integers between 1 and $N$ inclusively. Given two distributions $P_1$ and $P_2$ on a set $\+S$, we denote $d_{tv}$ as the total variation distance between the two distributions, i.e., $d_{tv}(P_1, P_2) = \| P_1 - P_2 \|_1/2$. We denote $d_{w,p}$ as the $p$-Wasserstein distance,
i.e., 
$
d_{w,p}(P_1,P_2)=(\inf_{c\in\+C} \E_{x,y\sim c}\|x-y\|^p)^{1/p}
$ where $\+C$ denotes the set of all couplings of $P_1$ and $P_2$. We note that $d_{tv}$ dominates $d_{w,p}$ when the support is bounded (see \cref{lem:tv-wass-general} for details):
\begin{equation}\label{eq:tv-wass}
d_{w,p}^p(P_1,P_2)\le \text{\rm diam}^p(\+S)\cdot d_{tv}(P_1,P_2)
\end{equation}
where $\text{\rm diam}(\+S)=\sup_{x,y\in\+S}\|x-y\|$ is the diameter of $\+S$.

\subsection{Finite-Horizon MDPs}

We consider a finite-horizon MDP with a vector-valued reward function, which is a tuple $M(\+X, \+A, r, P, H, \mu)$ where $\+X$ and $\+A$ are the state and action spaces, respectively, $P$ is the transition kernel, $r$ is the reward function, i.e.,  %
$r(x,a) \in \Delta([0, 1]^d)$ where $d\in \mathbb{Z}^+$, $H$ is the length of each episode, and $\mu\in \Delta(\mathcal{X})$ is the initial state distribution.
A policy is a mapping $\pi:\+X\rightarrow\Delta(\+A)$. We denote $z  \in [0,H]^d$ as the accumulative reward vector across $H$ steps, i.e., $z = \sum_{h=1}^H r_h$. Note that $z$ is a random vector whose distribution is determined by a policy $\pi$ and the MDP. We denote $Z^\pi \in \Delta([0,H]^d)$ as the distribution\footnote{{ Formally, they are called probability density functions in the continuous setting and probability mass functions in discrete settings, which are different from cumulative distribution functions.} } of the random variable $z$ under policy $\pi$. In this paper, we are interested in estimating $Z^\pi$ using offline data. We also define conditional distributions $Z^\pi_h(x,a) \in \Delta([0,H]^d)$ which is the distribution of the return under policy $\pi$ starting with state action $(x_h,a_h) := (x,a)$ at time step $h$. It is easy to see that $Z^\pi =\mathbb{E}_{x\sim \mu, a\sim \pi(x)} \left[ Z^\pi_1(x,a)\right]$. We define $d^\pi_h$ as the state-action distribution induced by policy $\pi$ at time step $h$, and $d^\pi = \sum_{h=1}^H d^\pi_h / H$ as the average state-action distribution induced by $\pi$. 

We denote the distributional Bellman operator \citep{morimura2012parametric} associated with $\pi$ as $\mathcal{T}^\pi$, which maps a conditional distribution to another conditional distribution: given a state-action conditional distribution $f \in \mathcal{X}\times\mathcal{A} \mapsto \Delta([0, H]^d)$, we have $ \mathcal{T}^\pi f  \in \mathcal{X}\times\mathcal{A}\mapsto \Delta([0,H]^d)$, such that for any $(x,a,z)$:
\begin{align*}
[\mathcal{T}^\pi f]&(z\given x,a) \\
&= \mathbb{E}_{r\sim r(x,a), x'\sim P(x,a), a'\sim \pi(x')} \left[ f\left(z- r| x',a'\right) \right].
\end{align*}
We can verify that $\mathcal{T}^\pi Z^\pi_{h+1} = Z^\pi_h$ for all $h$. %

\subsection{Discounted Infinite-Horizon MDPs}

The discounted infinite-horizon MDP is a tuple $M(\+X,\+A,r,P,\gamma,\mu)$. The return vector is defined as $z = \sum_{h=1}^{\infty} \gamma^{h-1} r_h$. We call $\gamma \in (0,1)$ the discount factor. The distribution of return $z$ is thus $Z^\pi\in\Delta([0,(1-\gamma)^{-1}]^d)$. We also define the conditional distribution $\-Z^\pi(x,a) \in \Delta([0,(1-\gamma)^{-1}]^d)$ which is the distribution of the return under policy $\pi$ starting with state action $(x,a)$. It is easy to see that $Z^\pi =\={E}_{x\sim \mu, a\sim \pi(x)} \left[\-Z^\pi(x,a)\right]$. The state-action distribution of a given policy $\pi$ is also defined in a discounted way: $d^\pi=(1-\gamma)^{-1}\sum_{h=1}^\infty\gamma^{h-1} d^\pi_h$ where $d^\pi_h$ is the state-action distribution induced by $\pi$ at time step $h$. The distributional Bellman operator maps a state-action conditional distribution $f\in\+X\times\+A\mapsto([0,(1-\gamma)^{-1}]^d)$ to $\+T^\pi f\in\+X\times\+A\mapsto([0,(1-\gamma)^{-1}]^d)$ for which
\begin{align*}
[\mathcal{T}^\pi &f](z\given x,a) \\
&= \mathbb{E}_{r\sim r(x,a), x'\sim P(x,a), a'\sim \pi(x')} \left[ f\left(\frac{z- r}{\gamma} \,\middle|\,x',a'\right) \right].
\end{align*}
for any $(x,a,z)$. We can verify that $\-Z^\pi$ is a fixed point of the distributional Bellman operator, i.e., $\+T^\pi \-Z^\pi =\-Z^\pi$.

\subsection{Offline Policy Evaluation Setup}

We consider estimating the distribution $Z^\pi$ using offline data which does not come from $\pi$ (i.e., off-policy setting). We assume we have a dataset $\mathcal{D} =\{ x_i, a_i, r_i, x'_i  \}_{i=1}^n$ that contains i.i.d. tuples, such that $x,a\sim \rho \in \Delta(\mathcal{X}\times\mathcal{A})$, $s'\sim P(\cdot | s,a)$, and $r \sim r(s,a)$. For finite-horizon MDPs, we randomly and evenly split $\+D$ into $H$ subsets, $\+D_1,\dots,\+D_H$, for the convenience of analysis. Each subset contains $n/H$ samples. For infinite-horizon MDPs, we split it into $T$ subsets in the same way. Here $T$ is the number of iterations which we will define later.

We consider learning distribution $Z^\pi$ via general function approximation. For finite-horizon MDPs, we denote $\mathcal{F}_h$ as a function class that contains state-action conditional distributions, i.e., $\mathcal{F}_h \subset \mathcal{X}\times\mathcal{A} \mapsto \Delta([0,H]^d)$, which will be used to learn $Z^\pi_h$. For infinite-horizon MDPs, we assume a function class $\mathcal{F} \subset \mathcal{X}\times\mathcal{A} \mapsto \Delta([0,(1-\gamma)^{-1}]^d)$.

\section{Fitted Likelihood Estimation}

In this section, we present our algorithm --- \emph{Fitted Likelihood Estimation} (FLE) for distributional OPE. \cref{alg:mle-ope} is for finite-horizon MDPs, and \cref{alg:mle-ope-inf} is for infinite-horizon MDPs.

\cref{alg:mle-ope} takes the offline dataset $\+{D}=\{\+D_h\}_{h=1}^H$  and the function class $\{ \mathcal{F}_h\}_{h=1}^H$ as inputs and iteratively performs Maximum likelihood estimation (MLE) starting from $H$ to time step $h=1$.  For a particular time step $h$, given $\hat f_{h+1}$ which is learned from the previous iteration,  FLE treats $\mathcal{T}^\pi \hat f_{h+1}$ as the target distribution to fit. To learn $\mathcal{T}^\pi \hat f_{h+1}$, it first generates samples from it (Line~\ref{line:generate}), which is doable as long as we can generate samples from the conditional distribution $\hat f_{h+1}(\cdot | x,a)$ given any $(x,a)$. Once we generate samples from $\mathcal{T}^\pi \hat f_{h+1}$, we fit $\hat f_{h}$ to estimate $\mathcal{T}^\pi \hat f_{h+1}$ by MLE (Line~\ref{line:mle}).
The algorithm returns $\hat f_1$ to approximate $Z^\pi_1$. To estimate $Z^\pi$, we can compute $\mathbb{E}_{x\sim \mu, a\sim \pi(x)} \hat f_1(x,a)$, recalling $\mu$ is the initial state distribution.

\cref{alg:mle-ope-inf} is quite similar to \cref{alg:mle-ope} but is for infinite-horizon MDPs, and it has two distinctions. First, we introduce the discount factor $\gamma$. Second, compared to \cref{alg:mle-ope} where we perform MLE in a backward manner (from $h=H$ to $1$), here we repeatedly apply MLE in a time-independent way. Particularly, it treats $\+T^\pi\^f_{t-1}$ as the target distribution to fit by MLE at round $t$. To finally estimate $Z^\pi$, we can compute $\mathbb{E}_{x\sim \mu, a\sim \pi(x)} \hat f_T(x,a)$.

To implement either algorithm, we need a function $f$ that has the following two properties: (1) it can generate samples given any state-action pair, i.e., $z\sim f(\cdot | x,a)$, and (2) given any triple $(x,a,z)$ we can evaluate the conditional likelihood, i.e., we can compute $f(z | x,a)$. Such function approximation is widely available in practice, including discrete histogram-based models,  Gaussian mixture models, Flow models \citep{dinh2014nice}, and diffusion model \citep{sohl2015deep}. Indeed, in our experiment, we implement FLE with Gaussian mixture models and diffusion models \citep{ho2020denoising}, both of which are optimized via MLE. \looseness=-1

\begin{algorithm}[t]
   \caption{Fitted Likelihood Estimation (FLE) for finite-horizon MDPs}
   \label{alg:mle-ope}
\begin{algorithmic}[1]
	\STATE {\bfseries Input:} dataset $\{\+D_h\}_{h=1}^H$ and function classes $\{\+F_h\}_{h=1}^H$
   	\FOR{$h=H,H-1,\dots,1$}
   		\STATE $\+D_h'=\emptyset$
   		\FOR{$x,a,r,x'\in\+D_h$}
			\IF{$h<H$}
				\STATE $a'\sim\pi(x')$, $y\sim\^f_{h+1}(\cdot\given x', a')$ \alglinelabel{line:generate}
				\STATE Set $z=r+y$
			\ELSE
				\STATE Set $z=r$
			\ENDIF
			\STATE $\+D'_h=\+D'_h\cup\{(x,a,z)\}$
		\ENDFOR
   		\STATE $\^f_h=\argmax_{f\in\+F_h}\sum_{(x,a,z)\in\+D'_h} \log f(z\given x,a)$\alglinelabel{line:mle}
	\ENDFOR
\end{algorithmic}
\end{algorithm}
\begin{algorithm}[t]
   \caption{Fitted Likelihood Estimation (FLE) for infinite-horizon MDPs}
   \label{alg:mle-ope-inf}
\begin{algorithmic}[1]
	\STATE {\bfseries Input:} dataset $\{\+D_t\}_{t=1}^T$ and function classes $\+F$
   	\FOR{$t=1,2,\dots,T$}
   		\STATE $\+D_t'=\emptyset$
   		\FOR{$x,a,r,x'\in\+D_t$}
			\STATE $a'\sim\pi(x')$
			\STATE $y\sim\^f_{t-1}(\cdot\given x', a')$
			\STATE $z=r+\gamma y$
			\STATE $\+D'_t=\+D'_t\cup\{(x,a,z)\}$
		\ENDFOR
   		\STATE $\^f_t=\argmax_{f\in\+F}\sum_{(x,a,z)\in\+D'_t} \log f(z\given x,a)$\alglinelabel{line:mle-inf}
	\ENDFOR
\end{algorithmic}
\end{algorithm}
Regarding computation, the main bottleneck is the MLE step (Line~\ref{line:mle} and \ref{line:mle-inf}). While we present it with a $\argmax$ oracle, in both practice and theory, an approximation optimization oracle is enough. In theory, as we will demonstrate, as long as we can find some $\hat f_{h}$ that exhibits good in-distribution generalization bound (i.e., $\mathbb{E}_{x,a\sim \rho} d_{tv}(\hat f_{h}(x,a), [\mathcal{T}^\pi \hat f_{h+1}](x,a))$ or $\mathbb{E}_{x,a\sim \rho} d_{tv}(\hat f_{t}(x,a), [\mathcal{T}^\pi \hat f_{t-1}](x,a))$ is small), then we can guarantee to have an accurate estimator for $Z^\pi$. Note that here $\rho$ is the training distribution for MLE, thus we care about in-distribution generalization. Thus our approach is truly a reduction to supervised learning: as long as the supervised learning procedure (in this case, MLE) learns a model with good in-distribution generalization performance, we can guarantee good prediction performance for FLE. Any advancements from training generative models via MLE  (e.g., better training heuristics and better models) thus can immediately lead to improvement in distributional OPE.

\begin{remark}[Comparison to prior models]\label{rmk:cate} 
    The categorical algorithm~\citep{bellemare2017distributional} works by minimizing the cross-entropy loss between the (projected) target distribution and the parametric distribution, which is equivalent to maximizing the likelihood of the parametric model.
\end{remark}

\begin{remark}[FQE as a special instance]\label{rmk:fqe}
When reward is only a scalar, and we use fixed-variance Gaussian distribution $f(\cdot | x,a) := \mathcal{N}( g(x,a), \sigma^2 )$ where $g:\mathcal{X}\times\mathcal{A}\mapsto [0,H]$, and $\sigma > 0$ is a fixed (not learnable) parameter,  MLE becomes a least square oracle, and FLE reduces to FQE --- the classic offline policy evaluation algorithm. 
\end{remark}

\section{Theoretical Analysis}\label{sec:analysis}

In this section, we present the theoretical guarantees of FLE. As a warm-up, we start by analyzing the performance of FLE for the finite-horizon setting (\cref{sec:finite-the}) where we bound the prediction error using total variation distance. Then we study the guarantees for the infinite-horizon discounted scenario in \cref{sec:infinite-the} where the prediction error is measured under $p$-Wasserstein distance. Note that from \cref{eq:tv-wass}, TD distance dominates $p$-Wasserstein distance, which indicates that our guarantee for the finite horizon setting is stronger. This shows an interesting difference between the two settings.  In addition, we present two concrete examples (tabular MDPs and linear quadratic regulators) in \cref{sec:example}. All proofs can be found in Appendix~\ref{app:proofs}. 

\subsection{Finite Horizon}\label{sec:finite-the}

We start by stating the key assumption for OPE, which concerns the overlap between  $\pi$'s distribution and the offline distribution $\rho$. \looseness=-1
\begin{assumption}[Coverage]\label{asm:cover}
    We assume there exists a constant $C$ such that for all $h\in[H]$ the following holds
    \begin{align*}
        \sup_{\substack{f_h\in\+F_h\\f_{h+1}\in\+F_{h+1}}}
        \frac
        {\E_{x,a\sim d^\pi_h}d^2_{tv}\left(f_h(x,a),[\+T^\pi f_{h+1}](x,a)\right)}
        {\E_{x,a\sim \rho}d^2_{tv}\left(f_h(x,a),[\+T^\pi f_{h+1}](x,a)\right)}
        \le C.
    \end{align*}
\end{assumption} 
The data coverage assumption is necessary for off-policy learning.  Assumption~\ref{asm:cover} incorporates the function class into the definition of data coverage and is always no larger than the usual density ratio-based coverage definition, i.e., $\sup_{h,x,a} d_h^\pi(x,a) / \rho(x,a)$ which is a classic coverage measure in offline RL literature (e.g., \citet{munos2008finite}). { This type of refined coverage is used in the regular RL setting \citep{xie2021bellman,uehara2021pessimistic}. %

Next, we present the theoretical guarantee of our approach under the \emph{assumption that the MLE can achieve good supervised learning-style in-distribution generalization bound}. Recall that in each iteration of our algorithm, we perform MLE to learn a function $\hat f_h$ to approximate the target $\mathcal{T}^\pi \hat f_{h+1}$ under the training data from $\rho$. By supervised learning style in-distribution generalization error, we mean the divergence $d_{tv}$ between $\hat f_h$ and the target $\mathcal{T}^\pi \^f_{h+1}$ under the \emph{ training distribution $\rho$}. Such an in-distribution generalization bound for MLE is widely studied in statistical learning theory literature \citep{van2000empirical,zhang2006}, and used in RL literature (e.g., \citet{agarwal2020flambe,uehara2021representation,zhan2022pac}). The following theorem demonstrates a reduction framework: as long as supervised learning MLE works, our estimator of $Z^\pi$ is accurate.

\begin{theorem}\label{thm:bell-error-to-final}
Under \cref{asm:cover}, suppose we have a sequence of functions $\^f_1,\dots,\^f_H:\+X\times\+A\mapsto\Delta([0,H]^d)$ and a sequence of values $\zeta_1,\dots,\zeta_H\in\=R$ such that
\begin{equation*}%
 \textstyle \bigg(\E_{x,a\sim\rho}\ d^2_{tv}\left(\^f_h(x,a), [\+T^\pi \^f_{h+1}](x,a)\right)\bigg)^{1/2}\le \zeta_h
\end{equation*}
holds for all $h\in[H]$.
Let our estimator $\^f\coloneqq\E_{x\sim\mu,a\sim\pi(x)}\^f_1(x,a)$. Then we have
\begin{align*}
    d_{tv}\left(\^f,Z^\pi\right)\le \sqrt{C}\sum_{h=1}^H\zeta_h.
\end{align*}
\end{theorem}
Here recall that $C$ is the coverage definition. Thus the above theorem demonstrates that when $\rho$ covers $d^\pi$ (i.e., $C < \infty$), small supervised learning errors (i.e., $\zeta_h$) imply small prediction error for distributional OPE.

Now to complete the picture, we provide some sufficient conditions where MLE can achieve small in-distribution generalization errors.
The first condition is stated below.%
\begin{assumption}[Bellman completeness]\label{asm:bell-comp}
We assume the following  holds:
\begin{equation*}
\max_{h\in[H], f\in \mathcal{F}_{h+1}}\min_{g\in \mathcal{F}_h} \mathbb{E}_{x,a\sim \rho}\ d_{tv}\big( g(x,a), [\mathcal{T}^\pi f](x,a)\big)  = 0. 
\end{equation*} We call the LHS of the above inequality \emph{inherent (distributional) Bellman error}.
\end{assumption}
This condition ensures that in each call of MLE in our algorithm, the function class $\mathcal{F}_h$ contains the target $\mathcal{T}^\pi \hat f_{h+1}$. It is possible to relax this condition to a setting where the inherent Bellman error is bounded by a small number $\delta$ (i.e., for MLE, this corresponds to agnostic learning where the hypothesis class may not contain the target, which is also a well-studied problem in statistical learning theory \citep{van2000empirical}). Here we mainly focus on the $\delta = 0$ case. 

The Bellman completeness assumption (or, more generally, inherent Bellman error being small) is standard in offline RL literature \citep{munos2008finite}. Indeed, in the regular RL setting, when learning with off-policy data, without such a Bellman completeness condition, algorithms such as TD learning or value iteration-based approaches (e.g., FQE) can diverge \citep{tsitsiklis1996analysis}, and the TD fixed solution can be arbitrarily bad in terms of approximating the true value (e.g., \citet{munos2003error,scherrer2010should,kolter2011fixed}). Since distributional RL generalizes regular RL, to prove convergence and provide an explicit sample complexity, we also need such a  Bellman completeness condition. %

The second condition is the bounded complexity of $\mathcal{F}_h$. A simple case is when $\mathcal{F}$ is discrete where the standard statistical complexity of $\mathcal{F}$ is $\ln ( | \mathcal{F}_h |)$. We show the following result for MLE's in-distribution generalization error. 
\begin{lemma}\label{lem:mle}
Assume $|\+F_h|<\infty$. For FLE (\cref{alg:mle-ope}), under \cref{asm:bell-comp}, MLEs have the following guarantee:
\begin{align*}
    \E_{x,a\sim\rho}
    d_{tv}^2\Big(\^f_h(x,a),[\+T^\pi\^f_{h+1}]&(x,a)\Big)\le\frac{4H}{n}\log(|\+F_h|H/\delta)
\end{align*}
for all $h\in[H]$ with probability at least $1-\delta$. %
\end{lemma}
For infinite hypothesis classes, we use bracketing number \citep{van2000empirical} to quantify the statistical complexities. 
\begin{definition}[Bracketing number]
Consider a function class $\+F$ that maps $\+X$ to $\=R$. Given two functions $l$ and $u$, the bracket $[l, u]$ is the set of all functions $f \in \+F$ with $l(x) \leq f(x) \leq u(x)$ for all $x \in \+X$. An $\epsilon$-bracket is a bracket $[l, u]$ with $\|l-u\| \leq \epsilon$. The bracketing number of $\+F$ w.r.t. the metric $\|\cdot\|$ denoted by $N_{[]}(\epsilon, \+F,\|\cdot\|)$ is the minimum number of $\epsilon$-brackets needed to cover $\+F$.
\end{definition}

We can bound MLE's generalization error using the bracket number of $\mathcal{F}$.
\begin{lemma}\label{lem:mle-general}
For FLE (Algorithm~\ref{alg:mle-ope}), under Assumption \ref{asm:bell-comp}, we have
\begin{align*}
    \E_{x,a\sim\rho} 
    d_{tv}^2&\Big(\^f_h(x,a),[\+T^\pi \^f_{h+1}](x,a)\Big)\\
    &\le
    \frac{10H}{n}\log \left(
	N_{[]}\big((nH^d)^{-1},\+F_h,\|\cdot\|_\infty\big)
	H/\delta\right)
\end{align*}
for all $h\in[H]$ with probability at least $1-\delta$. %
\end{lemma}

It is noteworthy that the logarithm of the bracketing number is small in many common scenarios. We offer several examples in Section~\ref{sec:example}. Previous studies have also extensively examined it (e.g., \citet{van2000asymptotic}).

With the generalization bounds of MLE, via \cref{thm:bell-error-to-final}, we can derive the following specific error bound for FLE.%
\begin{corollary}\label{cor:mle-main}
	Under Assumption~\ref{asm:cover} and \ref{asm:bell-comp}, for FLE (Algorithm~\ref{alg:mle-ope}), with probability at least $1-\delta$, we have
	\begin{align*}
   		d_{tv}\left(\^f,Z^\pi\right)\le \sqrt{C}\sum_{h=1}^H\sqrt{
    	\frac{4H}{n}\log(|\+F_h|H/\delta)}
	\end{align*}
	when $|\+F_h|<\infty$ for all $h\in[H]$, and 
	\begin{align*}
   		&d_{tv}\Big(\^f,Z^\pi\Big)\\
     &\le \sqrt{C}\sum_{h=1}^H\sqrt{
    	\frac{10H}{n}\log  \left( 
	N_{[]}\big((nH^d)^{-1},\+F_h,\|\cdot\|_\infty\big)
	H/\delta\right)}.
	\end{align*}
	for infinite function class $\+F_h$.
\end{corollary}

Overall, our theory indicates that if we can train accurate distributions (e.g., generative models) via supervised learning (i.e., MLE here), we automatically have good predictive performance on estimating $Z^\pi$. This provides great flexibility for designing special algorithms. %

\begin{remark}[Offline CVaR Estimation]\label{rmk:cvar}
	As a simple application, FLE can derive an estimator for the CVaR of the return under the test policy $\pi$. This is doable because CVaR is Lipschitz with respect to distributions in total variation distance, and thus our results can be directly transferred. See \cref{sec:extension} for details. Essentially, any quantity that is Lipschitz with respect to distributions in total variation distance can be estimated using our method and the error bound of FLE directly applies.
\end{remark}

\subsection{Infinite Horizon}\label{sec:infinite-the}

Next we introduce the theoretical guarantees of FLE for infinite horizon MDPs. Although the idea is similar, there is an obstacle: we can no longer obtain guarantees in terms of the total variation distance. This is perhaps not surprising considering that the distributional Bellman operator for discounted setting is \textit{not} contractive in total variation distance~\citep{bellemare2017distributional}. Fortunately, we found the Bellman operator is contractive under the Wasserstein distance measure. Note that the contractive result we established under Wasserstein distance is different from previous works~\citep{bellemare2017distributional,bdr2022, zhang2021distributional} in that these previous works consider the \textit{supremum} Wasserstein distance: $\sup_{x,a} d_{w,p}$, while our contractive property is measured under an \textit{average} Wasserstein distance: $(\E_{x,a\sim d^\pi} d^{2p}_{w,p})^{1/(2p)}$ which is critical to get a sample complexity bound for distributional OPE. More formally, the following lemma summarizes the contractive property. 
\begin{lemma}\label{lem:contractive}
The distributional Bellman operator is $\gamma^{1-1/(2p)}$-contractive under the metric $(\E_{x,a\sim d^\pi} d^{2p}_{w,p})^{1/(2p)}$, i.e., for any $f,f'\in\+X\times\+A\mapsto[0,(1-\gamma)^{-1}]^d$, it holds that
\begin{align*}
\bigg(\E_{x,a\sim d^\pi} &d^{2p}_{w,p}\left([\+T^\pi f](x,a), [\+T^\pi f'](x,a)\right)\bigg)^{\frac{1}{2p}}\\
&\leq
\gamma^{1-\frac{1}{2p}}\cdot
\left(\E_{x,a\sim d^\pi} d^{2p}_{w,p}\left(f(x,a), f'(x,a)\right)\right)^{\frac{1}{2p}}.
\end{align*}
\end{lemma}
We note that the contractive result in $\sup_{x,a} d_{w,p}$ does not imply the result in the above lemma, thus not directly applicable to the OPE setting.

Due to the dominance of total variation distance over Wasserstein distance on bounded sets (see \eqref{eq:tv-wass}), MLE's estimation error under total variation distance can be converted to Wasserstein distance. This allows us to derive theoretical guarantees for FLE under Wasserstein distance. To that end, we start again with the coverage assumption that is similar to \cref{asm:cover}. Note that we have replaced the total variation distance with the Wasserstein distance.

\begin{assumption}[Coverage]\label{asm:cover-inf}
    We assume there exists a constant $C$ such that the following holds
    \begin{align*}
        \sup_{\substack{f,f'\in\+F}}
        \frac
        {\E_{x,a\sim d^\pi}d^{2p}_{w,p}\left(f(x,a),[\+T^\pi f'](x,a)\right)}
        {\E_{x,a\sim \rho}d^{2p}_{w,p}\left(f(x,a),[\+T^\pi f'](x,a)\right)}
        \le C.
    \end{align*}
\end{assumption} 
As similar to \cref{thm:bell-error-to-final}, the following theorem states that as long as the supervised learning is accurate, our estimator of $Z^\pi$ will be accurate as well under $p$-Wasserstein distance.

\begin{theorem}\label{thm:bell-error-to-final-inf}
Under \cref{asm:cover-inf}, suppose we have a sequence of functions $\^f_1,\dots,\^f_T:\+X\times\+A\mapsto\Delta([0,(1-\gamma)^{-1}]^d)$ and an upper bound $\zeta\in\=R$ such that
\begin{equation*}
 \textstyle \bigg(\E_{x,a\sim\rho}\ d_{w,p}^{2p}\left(\^f_t(x,a), [\+T^\pi \^f_{t-1}](x,a)\right)\bigg)^{\frac{1}{2p}}\le \zeta
\end{equation*}
holds for all $t\in[T]$. Let our estimator $\^f\coloneqq\E_{x\sim\mu,a\sim\pi(x)}\^f_T(x,a)$. Then we have, for all $p\ge1$,
\begin{equation}\label{eq:geo-zeta_t}
	d_{w,p}\left(\^f, Z^\pi\right)
	\leq
	\frac{2C^{\frac{1}{2p}}}{(1-\gamma)^{\frac{3}{2}}}
	\cdot \zeta
	+\frac{\sqrt{d}\cdot\gamma^{\frac{T}{2}}}{(1-\gamma)^{\frac{3}{2}}}.
\end{equation}
\end{theorem}

The upper bound in \eqref{eq:geo-zeta_t} is actually a simplified version as we aim to present a cleaner result. For a more refined upper bound that has detailed $p$-dependent terms, please refer to \cref{thm:refined} in the appendix. For the first additive term in $\eqref{eq:geo-zeta_t}$, we will later demonstrate that the $\zeta$ obtained from MLE depends on $p^{-1}$ at an exponential rate. The second term is insignificant as it converges to zero at the rate of $\gamma^{T/2}$.

To proceed, we introduce the Bellman completeness assumption for infinite-horizon MDPs, a key condition for MLE to achieve small in-distribution generalization errors.

\begin{assumption}[Bellman completeness]\label{asm:bell-comp-inf}
We assume the following  holds:
\begin{equation*}
\max_{f\in \+F}\min_{g\in\+F} \mathbb{E}_{x,a\sim \rho}\ d_{w,p}\big( g(x,a), [\mathcal{T}^\pi f](x,a)\big)  = 0. 
\end{equation*} 
\end{assumption}

Similar to the previous result, when Bellman completeness holds and the function class has bounded complexity, MLE achieves small generalization error, as the following shows.

\begin{lemma}\label{lem:mle-inf}
For FLE (\cref{alg:mle-ope-inf}), under \cref{asm:bell-comp-inf}, by applying MLEs we have, for all $t\in[T]$,
\begin{align*}
	\E_{x,a\sim\rho}d_{w,p}^{2p}\bigg(\^f_t&(x,a),[\+T^\pi \^f_{t-1}](x,a)\bigg)\\
	&\leq\left(\frac{\sqrt{d}}{1-\gamma}\right)^{2p}\frac{4T}{n}\log(|\+F|T/\delta)
\end{align*}
when $|\+F|<\infty$, and
\begin{align*}
	&\E_{x,a\sim\rho}d_{w,p}^{2p}\left(\^f_t(x,a),[\+T^\pi \^f_{t-1}](x,a)\right)\\
	&\leq\left(\frac{\sqrt{d}}{1-\gamma}\right)^{2p}\frac{10T}{n}\log\left(N_{[]}\left(\frac{(1-\gamma)^d}{n},\+F,\|\cdot\|_\infty\right)T\Big/\delta\right)
\end{align*}
when $|\+F|=\infty$, with probability at least $1-\delta$. 
\end{lemma}

The multiplicative term $T$ in the upper bounds above comes from the data splitting (recall that we have split the dataset $\+D$ into $T$ subsets: $\+D_1,\dots,\+D_T$). A more careful analysis may be able to get rid of it, leading to a slightly better polynomial dependence on the effective horizon $1/(1-\gamma)$ in the final sample complexity bound. We leave this for future work.

In view of the above result, to derive the specific error bound of FLE, we need to choose an appropriate $T$ to make a good balance. The $T$ we choose is of the logarithmic order. It is shown in the corollary below.
\begin{corollary}\label{cor:mle-main-inf}
We define 
\begin{align*}
	\iota=
	\begin{cases}
	\log(|\+F|/\delta),
	&\text{if\quad}|\+F|<\infty;\\
	\log\left(N_{[]}\left(\frac{(1-\gamma)^d}{n},\+F,\|\cdot\|_\infty\right)\Big/\delta\right),
	&\text{if\quad}|\+F|=\infty.
	\end{cases}
\end{align*}
Then under Assumption~\ref{asm:cover-inf} and \ref{asm:bell-comp-inf}, for FLE (Algorithm~\ref{alg:mle-ope-inf}), if we pick 
\begin{equation*}
	T=\log\left(C^{\frac{1}{2p}}\cdot\iota^{\frac{1}{2p}}\cdot\left(1-\gamma^{\frac{1}{2}}\right)^{-1}\cdot n^{-\frac{1}{2p}}\right)\Big/\log\left(\gamma^{1-\frac{1}{2p}}\right)
\end{equation*}
then with probability at least $1-\delta$, we have
\begin{align*}
 	d_{w,p}\left(\^f,Z^\pi\right)\le 
	\w~O\left(
	\frac
	{C^{\frac{1}{2p}}\cdot \iota^{\frac{1}{2p}}\cdot \sqrt{d}}
	{(1-\gamma)^{\frac{5}{2}}}
	\cdot n^{-\frac{1}{2p}}
	\right)
\end{align*}
where $\^f\coloneqq\E_{x\sim\mu,a\sim\pi(x)}\^f_T(x,a)$.
\end{corollary}
The above upper bound depends on $n^{-1/(2p)}$, which seems unsatisfactory, especially when $p$ is large. However, we believe that it is actually tight since the previous study has shown that the minimax rate of estimating $d_{w,p}$ using i.i.d samples from the given distribution is around $O(n^{-1/(2p)})$~\citep{singh2018minimax}. More formally, given a distribution $Q$ and $n$ i.i.d samples from $Q$, any algorithm that maps the $n$ i.i.d samples to a distribution $\hat Q$, must have $d_{w,p}(\hat Q, Q) = \widetilde\Omega( n^{-1/(2p)})$ in the worst case. Note that distributional OPE is strictly harder than this problem. 
\section{Simulation}

\begin{figure}[htb]
\begin{center}
\centerline{\includegraphics[width=0.97\columnwidth]{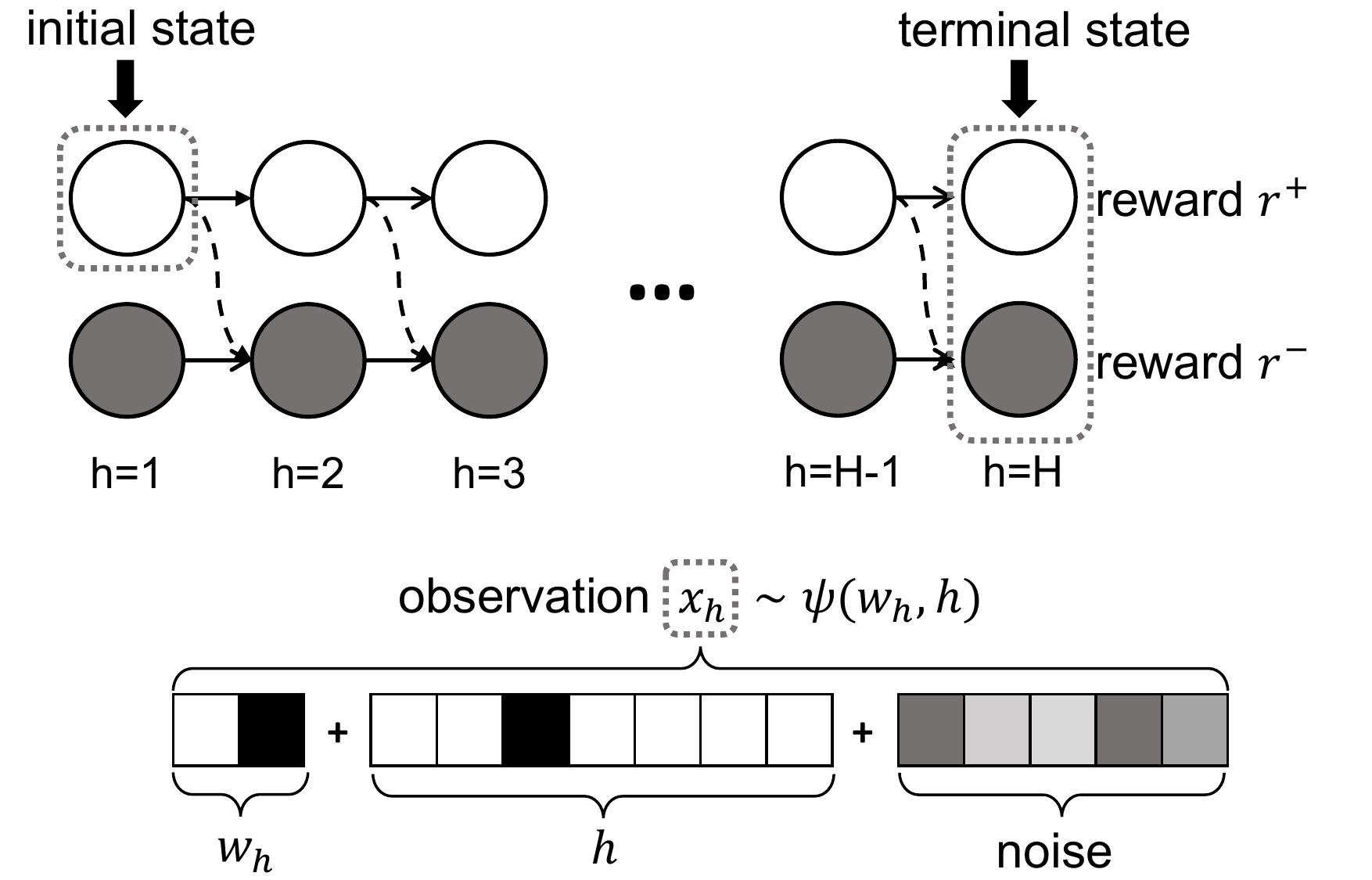}}
\caption{Visualization of the combination lock. The dotted lines denote transiting from good states (white) to bad states (gray). Once the agent transits to a bad state, it stays there forever. The observation is composed of three parts: one-hot encoding of the latent state $w_h$, one-hot encoding of the step $h$, and random noise.}
\label{fig:env}
\end{center}
\vspace{-20pt}
\end{figure}
In this section, we show the empirical performance of two instances of FLE: \emph{GMM-FLE} and \emph{Diff-FLE}.  
The GMM-FLE uses conditional Gaussian mixture models for $\mathcal{F}$,  for which the weights and the mean and covariance of Gaussians are all learnable. 
For Diff-FLE, we model the distribution $f(\cdot\,|\,x,a)$ as a conditional diffusion probabilistic model \citep{sohl2015deep}. The implementation is based on DDPM~\citep{ho2020denoising}. We elaborate on other components of the experiments below. See \cref{app:exp-details} for implementation details and a full list of results.

\textbf{The combination lock environment.}  
The combination lock consists of two chains. One of the chains is good, while the other is bad. The agent wants to stay on the good chain, for which the only approach is to take the unique optimal action at all time steps. See \cref{fig:env} for an illustration. Mathematically, the combination lock is a finite-horizon MDP of horizon $H$. There are two latent states $w_h\in\{0,1\}$. At any time step $h\in[H]$, there is only one optimal action $a^\star_h$ among $A$ actions. If the agent is in the latent state $w_h=0$ and takes $a^\star_h$, it transits to $w_{h+1}=0$, and otherwise transits to $w_{h+1}=1$. If it is already in $w_h=1$, no matter what action it takes, it transits to $w_{h+1}=1$. When $h=H$, it receives a random reward $r^+$ if $w_H=0$; otherwise, it gets $r^-$. The agent cannot observe the latent state $w_h$ directly. Instead, the observation it receives, $\psi(w_h,h)$, is the concatenation of one-hot coding of the latent state $w_h$ and the current time step $h$, appended with Gaussian noise. This environment has been used in prior works \citep{misra2020kinematic,zhang2022efficient} where it was shown that standard deep RL methods struggle due to the challenges from exploration and high-dimensional observation.

\textbf{Test policy.} 
The test policy is stochastic: it takes a random action with probability $\epsilon$ and takes the optimal policy otherwise. In all experiments, we set $\epsilon=1/7$. 

\textbf{Offline data generation.} 
The offline dataset is generated uniformly. Specifically, for each time step $h\in[H]$ and each latent state $w_h\in\{0,1\}$, we first randomly sample 10000 observable state $\phi(w_h)$. Then for each of them, we uniformly randomly sample action and perform one step simulation. It is clear that the offline data distribution here satisfies the coverage assumption (Assumption~\ref{asm:cover}). %

\subsection{One-Dimensional Reward}\label{sec:exp-1d}

To compare to classic methods such as the categorical algorithm~\citep{bellemare2017distributional} and quantile TD~\citep{dabney2018distributional}, we first run experiments with a 1-d reward. Specifically, we have $r^+\sim\+N(1,0.1^2)$ and $r^-\sim\+N(-1,0.1^2)$. The horizon is $H=20$. 

The categorical algorithm discretizes the range $[-1.5,1.5]$ using 100 atoms. For quantile TD, we set the number of quantiles to 100 as well. The GMM-FLE uses 10 atomic Gaussian distributions, although eventually, only two are significant. See Appendix~\ref{app:exp-details} for a detailed description of implementations. We plot the PDFs  $\E_{x\sim\psi(0,h)} \hat f_h(x,a^\star_h)$ (here $0$ denotes the good latent state in $h$) learned by different methods in \cref{fig:1d}, at three different time steps. As we can see, GMM-FLE in general fits the ground truth the best. 

We compute the approximated $d_{tv}$ between the learned distribution and the true one. Ideally, we want to compute
$d_{tv}(
	\E_{x\sim\psi(0,h)}\^f_h(x,a^\star_h),
	\E_{x\sim\psi(0,h)}Z^\pi_h(x,a^\star_h))$.
However, since obtaining the density of certain models is impossible (e.g., Diff-FLE) and certain other models have only discrete supports, we use an approximated version: we sample $20k$ points from each distribution, construct two histograms, and calculate $d_{tv}$ between the two histograms. %
The results are shown in \cref{tab:1d}. Again, GMM-FLE achieves the smallest total variation distance. This intuitively makes sense since the ground truth return is a mixture of Gaussians. Moreover, we notice that GMM-FLE, Diff-FLE, and categorical algorithms achieve significantly better performance than the quantile regression TD algorithm. This perhaps is not surprising because our theory has provided performance guarantees for those three algorithms under $d_{tv}$ (recall that the categorical algorithm can be roughly considered a specification of FLE, see Remark~\ref{rmk:cate}), while it is unclear if quantile regression TD can achieve similar guarantees in this setting. \edit{In addition, we also compute the approximated $d_{w,1}$ ($1$-Wasserstein distance) between the learned distribution and the true one. Please refer to Appendix~\ref{app:full-exp-results} and Table~\ref{tab:1d-wass-full} for details.}
\looseness=-1

\begin{table}[htb]
\begin{center}
\begin{small}
\resizebox{\columnwidth}{!}{
\begin{tabular}{ccccc}
\toprule
$h$ & Cate Alg & Quan Alg & Diff-FLE & GMM-FLE \\
\midrule
1	&	0.071 	$\pm$	0.015 	&		0.603 	$\pm$	0.011 	&		0.292 	$\pm$	0.073 	&		\textbf{0.039 	$\pm$	0.004} 	\\
10	&	0.079 	$\pm$	0.017 	&		0.494 	$\pm$	0.018 	&		0.234 	$\pm$	0.043 	&		\textbf{0.044 	$\pm$	0.012} 	\\
19	&	0.078 	$\pm$	0.011 	&		0.167 	$\pm$	0.019 	&		0.109 	$\pm$	0.031 	&		\textbf{0.018 	$\pm$	0.008} 	\\
\bottomrule
\vspace{-10pt}
\end{tabular}
}
\caption{Approximated $d_{tv}$ between $\E_{x\sim\psi(0,h)}\^f_h(x,a^\star_h)$ and $\E_{x\sim\psi(0,h)}Z^\pi_h(x,a^\star_h)$ in the 1-d case. The means and standard errors are computed via five independent runs.}
\label{tab:1d}
\end{small}
\end{center}
\vspace{-28pt}
\end{table}

\noindent
\begin{figure}[t]
\begin{center}
\centerline{\includegraphics[width=0.98\columnwidth]{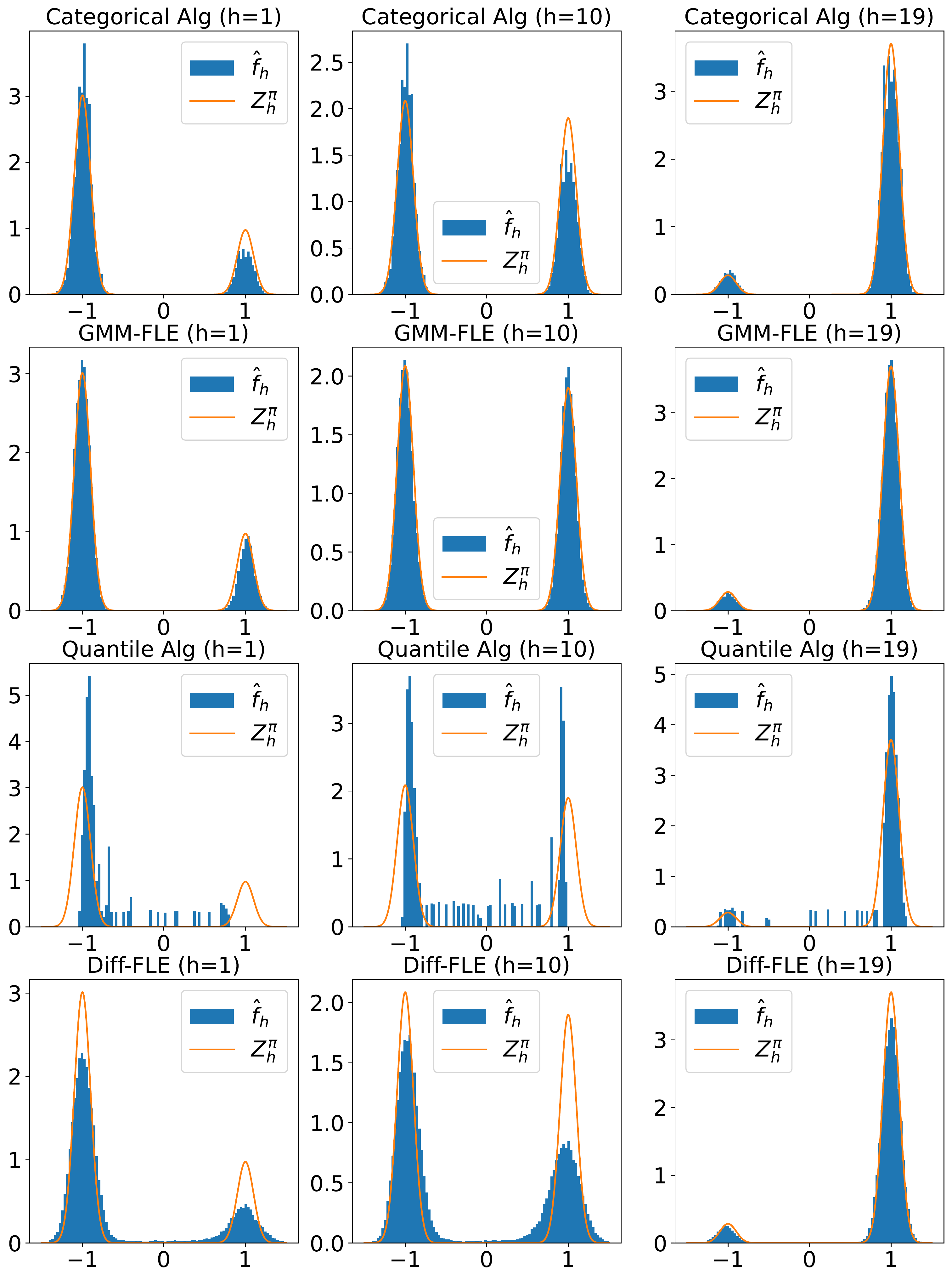}}
\caption{Plots  of $\E_{x\sim\psi(0,h)}\^f_h(x,a^\star_h)$ and $\E_{x\sim\psi(0,h)}Z^\pi_h(x,a^\star_h)$.  The histograms are generated via 50k samples.}
\label{fig:1d}
\end{center}
\vspace{-23pt}
\end{figure}
\begin{figure}[t]
\begin{center}
\centerline{\includegraphics[width=\columnwidth]{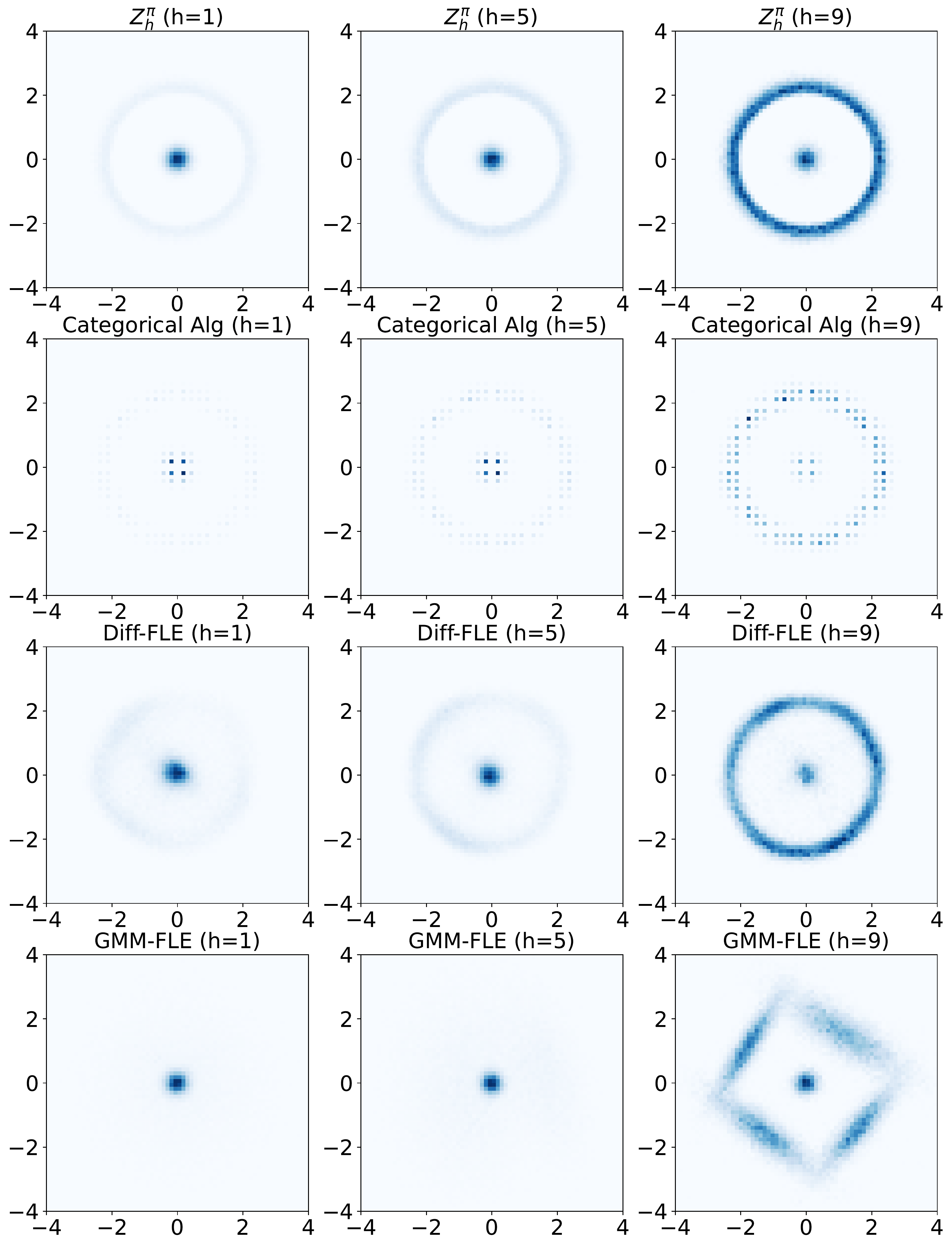}}
\caption{Plots of $\E_{x\sim\psi(0,h)}\^f_h(x,a^\star_h)$ (generated via 50k samples), and the ground truth 
$\E_{x\sim\psi(0,h)}Z^\pi_h(x,a^\star_h)$ (top row).}
\label{fig:2d}
\end{center}
\vspace{-31pt}
\end{figure}
\subsection{Two-Dimensional Reward}\label{sec:exp-2d}

We also conducted experiments on two-dimensional rewards where $r^+$ is sampled from a ring in $\=R^2$ of radius 2 and $r^-$ follows a Gaussian centered at the origin. The horizon is $H=10$. The categorical algorithm discretizes the range $[-4,4]^2$  into $30$ atoms per dimension (totaling $900$ atoms). Although the 2-d version of the categorical algorithm is not introduced in the original paper~\citep{bellemare2017distributional}, the extension is intuitive. The GMM-FLE employs 30 atomic Gaussian distributions, but only up to six prove significant in the end. We note that extending quantile regression TD to multi-dimensional rewards is not straightforward.

\begin{table}[b]
\vspace{-5pt}
\begin{center}
\begin{small}
\begin{sc}
\begin{tabular}{ccccc}
\toprule
$h$ & Cate Alg & Diff-FLE & GMM-FLE \\
\midrule
1	&	0.483 	$\pm$	0.003 		&	\textbf{0.357 	$\pm$	0.031} 		&	0.438 	$\pm$	0.008 	\\
5	&	0.466 	$\pm$	0.001 		&	\textbf{0.310 	$\pm$	0.019} 		&	0.493 	$\pm$	0.050 	\\
9	&	0.453 	$\pm$	0.001 		&	\textbf{0.207 	$\pm$	0.014} 		&	0.502 	$\pm$	0.094 	\\
\bottomrule
\end{tabular}
\end{sc}
\end{small}
\end{center}
\caption{Approximated $d_{tv}$ between $\E_{x\sim\psi(0,h)}\^f_h(x,a^\star_h)$ and $\E_{x\sim\psi(0,h)}Z^\pi_h(x,a^\star_h)$ in the 2-d case. The means and standard errors are computed via five independent runs.}
\label{tab:2d}
\end{table}

We plotted the 2-d visualization of the learned distribution in \cref{fig:2d} and computed the approximated TV distance using the same method as in the 1-d case, which is shown in \cref{tab:2d}.
Diff-FLE achieves the smallest TV error (\cref{tab:2d}) and captures the correlation among dimensions (i.e., see \cref{fig:2d} where Diff-FLE captures the ring structures in all steps).  
However, the GMM-FLE also doesn't perform well since it is hard for vanilla GMM with a finite number of mixtures to capture a ring-like data distribution. 
The two-dimensional categorical algorithm performed badly as well, even though it uses a larger number of atoms (recall that for the 1-d case it only uses 100 atoms and already achieves excellent performance), implying that it suffers from the curse of dimensionality statistically, i.e., explicitly discretizing the 2-d return space evenly can fail to capture the underlying data structure (e.g., in our ring example, data actually approximately lives in a sub-manifold). Moreover, the training is also significantly slower. In our implementation, we found that running the 2-d categorical algorithm with $100^2$ atoms is about 100 times slower than running the 1-d algorithm with $100$ atoms, while the training time of Diff-FLE and GMM-FLE does not change too much.

\section{Discussion and Future Work}

We proposed Fitted Likelihood Estimation (FLE), a simple algorithm for distributional OPE with multi-dimensional rewards. FLE conducts a sequence of MLEs and can incorporate any state-of-the-art generative models trained via MLE. Thus, FLE is scalable to the setting where reward vectors are high-dimensional. Theoretically, we showed that the learned distribution is accurate under total variation distance and $p$-Wasserstein distance for the finite-horizon and infinite-horizon discounted setting, respectively. In practice, we demonstrated its flexibility in utilizing generative models such as GMMs and diffusion models. 

Our work may offer several promising avenues for future research in distributional RL. One immediate direction is to adapt our algorithms to the policy optimization. Another direction is the development of more efficient algorithms that can work in more complex environments.

\section*{Acknowledgement}
WS acknowledges funding support from NSF IIS-2154711. 
\edit{We thank Mark Rowland for the useful discussion on the Bellman completeness.}

\bibliography{references.bib}
\bibliographystyle{icml2023}

\newpage
\appendix
\onecolumn

\section{Offline CVaR Evaluation}\label{sec:extension}

We consider estimating the CVaR of $Z^\pi$ with $d=1$. %
Given a threshold $\tau \in (0,1)$, the $\text{CVaR}_\tau$ of $Z^\pi$ is defined as (assuming finite-horizon MDPs): \looseness=-1
\begin{align*}
\text{CVaR}_\tau(Z^\pi) := \max_{b\in [0,H]} \left( b - \frac{1}{\tau} \mathbb{E}_{z\sim Z^\pi} \max\left\{ b - z, 0\right\} \right).
\end{align*} 
CVaR intuitively measures the expected value of the random variable belonging to the tail part of the distribution and is often used as a risk-sensitive measure. The following lemma shows that $\text{CVaR}_{\tau}(Z^{\pi})$ is Lipschitz continuous with respect to metric $d_{tv}$ and the Lipschitz constant is $2H / \tau$. 
\begin{lemma}\label{lem:cvar-lip}
	Let $f, f'\in\Delta([0,H])$ be two densities. Then we have
	\begin{align*}
		\text{\rm CVaR}_\tau(f)-\text{\rm CVaR}_\tau(f')
		\leq
		\frac{2H}{\tau}\cdot d_{tv}(f,f').
	\end{align*}
\end{lemma}

\begin{proof}
Let $f, f'\in\Delta([0,H])$ denote two densities. Then we have
\begin{align*}
&\text{CVaR}_\tau(f)-\text{CVaR}_\tau(f')\\
=&\max_{b\in [0,H]} \left( b - \frac{1}{\tau} \mathbb{E}_{z\sim f} \max\left\{ b - z, 0\right\} \right)-\max_{b\in [0,H]} \left( b - \frac{1}{\tau} \mathbb{E}_{z\sim f'} \max\left\{ b - z, 0\right\} \right)\\
\le& \left( b_0 - \frac{1}{\tau} \mathbb{E}_{z\sim f} \max\left\{ b_0 - z, 0\right\} \right)-\left( b_0 - \frac{1}{\tau} \mathbb{E}_{z\sim f'} \max\left\{ b_0 - z, 0\right\} \right)\\
=&\frac{1}{\tau}\Big(\mathbb{E}_{z\sim f'} \max\left\{ b_0 - z, 0\right\}-\mathbb{E}_{z\sim f} \max\left\{ b_0 - z, 0\right\} \Big)\\
=&\frac{1}{\tau}\int_{[0,H]} \big(f'(z)-f(z)\big) \max\{b_0-z,0\}\d z\\
\le&\frac{H}{\tau}\int_{[0,H]} \big|f'(z)-f(z)\big|\d z\\
\le&\frac{2H}{\tau} d_{tv}(f,f')
\end{align*}
where the first inequality holds by picking $b_0=\argmax_{b\in [0,H]} \left( b - \frac{1}{\tau} \mathbb{E}_{z\sim f} \max\left\{ b - z, 0\right\} \right)$.
\end{proof}

Thus using our bound from \cref{cor:mle-main}, we get:
\begin{align*}
    \left\lvert \text{CVaR}_{\tau}( Z^{\pi} ) - \text{CVaR}_{\tau}( \hat f)  \right\rvert
    \leq \frac{4 C^{1/2}H^{2.5}}{\tau} \sqrt{ \frac{ \log(\max_h |\mathcal{F}|_h / \delta) }{ n  }  }, 
\end{align*} with probability at least $1-\delta$. %

\section{Examples}\label{sec:example}

In this section, we discuss two examples: one is tabular MDPs, and the other is Linear Quadratic Regulators. For simplicity of presentation, we focus on scalar rewards and finite horizon. %

\subsection{Tabular MDPs}
We consider tabular MDP (i.e., $|\mathcal{X}|$ and $|\mathcal{A}|$ are finite) with continuous known reward distributions. Specifically, we consider the sparse reward case where we only have a reward at the last time step $H$ and have zero rewards at time step $h<H$. For each $(x,a)$, Denote $r_H(x,a) \in \Delta([0,1])$. %

Note that in this setup, via induction, it is easy to verify that for any $h, x,a $, $Z^\pi_h(\cdot | x,a)$ is a mixture of the distributions $\{r_H(x,a): x\in \mathcal{X}, a\in \mathcal{A}\}$, i.e., for any $h, x,a$, there exists a probability weight vector $w\in \Delta(|\mathcal{X}| |\mathcal{A}|)$, such that $Z^\pi_h(\cdot | x, a) = \sum_{x',a' \in \mathcal{X}\times\mathcal{A}} w(x',a') r_H(\cdot|x',a')$. Note that the parameters $w(x,a)$ are unknown due to the unknown transition operator $P$, and need to be learned. Thus, in this case, we can design function class $\mathcal{F}_h$ as follows:  \looseness=-1
\begin{align*}
	\mathcal{F}_h = \bigg\{f(\cdot|x,a)=\sum_{x',a' \in \mathcal{X}\times\mathcal{A}} w_{x,a}(x',a') r_H(\cdot|x',a'):&\\
	\big\{w_{x,a}\in \Delta(|\mathcal{X}| |\mathcal{A}|)\big\}_{x,a\in\+X\times\+A} \bigg\}.&
\end{align*}
It is not hard to verify that $\{\mathcal{F}_h\}_{h=1}^H$ does satisfy the Bellman complete condition. The log of the bracket number of $\mathcal{F}_h$ is polynomial with respect to $|\mathcal{X}||\mathcal{A}|$.
\begin{lemma}\label{lem:tabular}
In the above example, the complexity of $\+F_h$ in bounded: $\log N_{[]}(\epsilon,\+F_h,\|\cdot\|_\infty) \le O(|\mathcal{X}|^2|\mathcal{A}|^2\log(r_\infty|\+X||\+A|/\epsilon))$ where $r_\infty\coloneqq\|r_H\|_\infty$. 
\end{lemma} 

Thus \cref{alg:mle-ope} is capable of finding an accurate estimator of $Z^\pi$ with sample complexity scaling polynomially with respect to the size of the state and action spaces and horizon.

\subsection{Linear Quadratic Regulator}\label{sec:lqr}

The second example is LQR. We have $\mathcal{X} \subset \mathbb{R}^{d_x}, \mathcal{A} \subset \mathbb{R}^{d_a}$. 
\begin{align*}
&x_{h+1} = A x_h + B a_h , \\
&r(x_h,a_h) = - ( x_h^\top Q x_h + a_h^\top R a_h) + \varepsilon
\end{align*}
where $\varepsilon \sim \mathcal{N}(0, \sigma^2)$.  %
Since the optimal policy for LQR is a linear policy, we consider evaluating a linear policy $\pi(x) := K x$ where $K\in \mathbb{R}^{d_a\times d_x}$.  \looseness=-1
For  this linear policy, $Z^\pi_h(\cdot | x,a)$ is a Gaussian distribution, i.e., $Z^\pi_h(\cdot | x,a) = \mathcal{N}(\mu_h(x,a), \sigma_h(x,a))$, where $\mu_h(x,a)$ and $\sigma_h(x,a)$ has closed form solutions.

\begin{lemma}\label{lem:lqr-closed-form}
For LQR defined above, $\mu_h(x,a)$ and $\sigma_h(x,a)$ has the following closed form solutions
\begin{align*}
	\mu_h(x,a)=&-(Ax+Ba)^\top U_{h+1}(Ax+Ba)\\
                    &-x^\top Q x-a^\top R a,\\
	\sigma^2_h(x,a)=&(H-h+1)\sigma^2
\end{align*}
where we denote $U_h=\sum_{i=h}^H((A+BK)^{i-h-1})^\top(Q+K^\top RK)(A+BK)^{i-h-1}$. 
\end{lemma}

Thus our function class $\mathcal{F}_h$ can be designed as follows:
\begin{align*}
	\+F_h=\Big\{
	f(\cdot|x,a)=
	\+N\big(\cdot\,\big|\,
	x^\top M_1 x
	+
	a^\top M_2 x
	+
	a^\top M_3 a
	, &\\
	(H-h+1)\sigma^2
	\big),\ 
	\forall M_1,M_2,M_3
	\Big\}&
\end{align*}
We can show that this function class satisfies Bellman completeness. {Furthermore, here, we can refine $C$ in Assumption~\ref{asm:cover} to a relative condition number following the derivation in \citet{uehara2021pessimistic}. More specifically, $C$ is %
    $\sup_{w\neq 0,h }\frac{w^{\top}\mathbb{E}_{d^{\pi}_h}[\phi(x,a)\phi^{\top}(x,a) ]w}{w^{\top}\mathbb{E}_{\rho}[\phi(x,a)\phi^{\top}(x,a)]w}$ 
where $\phi(x,a)=(x^{\top},a^{\top})^{\top}\otimes (x^{\top},a^{\top})^{\top} $ is a quadratic feature and $\otimes$ is the Kronecker product.} Under some regularity assumption (i.e., the norms of $M_1, M_2, M_3$ are bounded, which is the case when the dynamical system induced by the linear policy is stable), this function class has bounded statistical complexity. 

\begin{lemma}\label{lem:lqr-complextiy}
We assume there exist parameters $m_x,m_a,m_1,m_2,m_3$ for which $\|x\|_2\leq m_x$ for all $x\in\+X$ and $\|a\|_2\leq m_a$ for all $a\in\+A$, and $\|M_i\|_{\rm F}\leq m_i$ for $i=1,2,3$. Then we have
\begin{equation*}
	\log N_{[]}(\epsilon,\+F_h,\|\cdot\|_\infty)
	\leq
	{\rm Poly}\left(d_x,d_a,\log \frac{m_xm_am_1m_2m_3}{\epsilon\sigma}\right).
\end{equation*}
	
\end{lemma}

\edit{
It is unclear if quantile regression TD or categorical TD can achieve meaningful guarantees on LQR in general, because it is unclear how to design a function class that has bounded complexity and satisfies Bellman completeness. To be specific, the function class for quantile/categorical TD needs to satisfy the following two conditions on Bellman completeness:
(1) the errors incurred in the projection step is bounded, i.e., $\max_{f\in\+F} d (\+T^\pi f, \prod \+T^\pi f)$ is bounded (where $\prod$ denotes the projection onto the desired categorical/quantile finite support and $\+F$ is a subset of the space of return-distribution functions with the categorical/quantile support), and (2) the projected function class has zero (or low) inherent Bellman error, i.e., $\max_{f\in\+F} \min_{g\in\+F}  d(\prod \+T^\pi f, g) \approx0$. Under these conditions, the resulting algorithm may converge with bounded fixed point error as shown by \citet{rowland2018analysis,rowland2023analysis}. 
However, it is worth noting that the convergence rate will depend on the complexity of the function class and it is still unclear how to design such a function class with bounded polynomial complexity for LQR. Naively discretizing the state space will not work since the complexity will then depend on the dimensionality exponentially. The designing of such function classes is an interesting future research direction. 

However, we note that in general, even for regular RL, it is possible that TD-based algorithms may diverge without Bellman completeness in the off-policy setting, and TD fixed point solutions can be arbitrarily bad.
}
\section{Supporting Lemmas} 

\subsection{Maximum Likelihood Estimation}

In this section, we adapt the theoretical results of MLE \citep{agarwal2020flambe} to more general versions. We will follow the notation in Appendix E of~\citet{agarwal2020flambe} and restate the setting here for completeness.

We consider a sequential conditional probability estimation problem. Let $\+X$ and $\+Y$ denote the instance space and the target space, respectively. We are given a function class $\mathcal{F}:(\mathcal{X} \times \mathcal{Y}) \rightarrow \mathbb{R}$ with which we want to model the true conditional distribution $f^{\star}$. To this end, we are given a dataset $D:=\left\{\left(x_i, y_i\right)\right\}_{i=1}^n$, where $x_i \sim \mathcal{D}_i$ and $y_i \sim p\left(\cdot \mid x_i\right)=f^\star(x,\cdot)$. 

We only assume that there exists $f_i^\star$ for each $i\in[n]$ such that $\E_{x\sim\+D_i}d_{tv}(f^\star_i(x),f^\star(x))=0$. Note that this assumption only considers $x$ on the support of $\+D_i$ and is thus weaker than saying $f^\star\in\+F$.

For the data generating process, we assume the data distribution $\mathcal{D}_i$ is history-dependent, i.e., it can depend on the previous samples: $x_1,y_1,\dots,x_{i-1},y_{i-1}$. 

Let $\+D'=\{(x_i',y_i')\}_{i=1}^n$ denote the tangent sequence which is generated by $x_i'\sim \+D_i$ and $y_i'\sim p(\cdot\mid x'_i)$. The tangent sequence is independent when conditioned on $\+D$.

\begin{lemma}[Adapted version of Lemma 25~\citep{agarwal2020flambe}]\label{lem:convert-1norm}
    Let $f_1\in\+X\mapsto\Delta(\+Y)$ be a conditional probability density and $f_2\in\+X\times\+Y\mapsto\=R_{\ge0}$ (satisfying $\int_\+Y f_2(x,y)\d y\le s$ for all $x\in\+X$). Let $\+D\in\Delta(\+X)$ be any distribution. Then, we have
    \begin{equation*}
        \E_{x\sim\+D}\left(\int_\+Y\left|f_1(x,y)-f_2(x,y)\right|\d y\right)^2\le(2+2s)\left((s-1)-2\log\E_{x\sim\+D,y\sim f_1(x,\cdot)}\exp\left(-\frac{1}{2}\log\left(f_1(x,y)/f_2(x,y)\right)\right)\right).
    \end{equation*}
\end{lemma}
\begin{proof}[Proof of Lemma \ref{lem:convert-1norm}]
First, we have
\begin{align*}
    &\E_{x\sim\+D}\left(\int_\+Y\left|f_1(x,y)-f_2(x,y)\right|\d y\right)^2
    =
    \E_{x\sim\+D}\left(\int_\+Y\left|\sqrt{f_1(x,y)}-\sqrt{f_2(x,y)}\right|\left(\sqrt{f_1(x,y)}+\sqrt{f_2(x,y)}\right)\d y\right)^2\\
    \le&
    \E_{x\sim\+D}{\int_\+Y\left(\sqrt{f_1(x,y)}-\sqrt{f_2(x,y)}\right)^2\d y}\cdot
    {\int_\+Y\left(\sqrt{f_1(x,y)}+\sqrt{f_2(x,y)}\right)^2\d y}\\
    =&
    \E_{x\sim\+D}{\int_\+Y\left(\sqrt{f_1(x,y)}-\sqrt{f_2(x,y)}\right)^2\d y}\cdot2\int_\+Y\big(f_1(x,y)+f_2(x,y)\big)\d y-\int_\+Y\left(\sqrt{f_1(x,y)}-\sqrt{f_2(x,y)}\right)^2\d y\\
    =&
    \E_{x\sim\+D}{\int_\+Y\left(\sqrt{f_1(x,y)}-\sqrt{f_2(x,y)}\right)^2\d y}\cdot2\int_\+Y\big(f_1(x,y)+f_2(x,y)\big)\d y\\
    \le&
    \underbrace{\E_{x\sim\+D}\int_\+Y\left(\sqrt{f_1(x,y)}-\sqrt{f_2(x,y)}\right)^2\d y}_{\rm(*)}\cdot
    (2+2s).
\end{align*}
where the first inequality holds for Cauchy–Schwarz inequality.
For ${\rm(*)}$, we have
\begin{align*}
    {\rm(*)}=&\E_{x\sim\+D}\int_\+Y\left(\sqrt{f_1(x,y)}-\sqrt{f_2(x,y)}\right)^2\d y
    \le(s-1)+2-2\E_{x\sim\+D}\int_\+Y \sqrt{f_1(x,y)f_2(x,y)}\d y\\
    =&(s-1)+2\left(1-\E_{x\sim\+D}\int_\+Y \sqrt{f_1(x,y)f_2(x,y)} \d y\right)
    \le(s-1)-2\log\left(\E_{x\sim\+D}\int_\+Y \sqrt{f_1(x,y)f_2(x,y)} \d y\right)\\
    \le&(s-1)-2\log\E_{x\sim\+D,y\sim f_1(x,\cdot)}\sqrt{f_2(x,y)/f_1(x,y)}\\
    =&(s-1)-2\log\E_{x\sim\+D,y\sim f_1(x,\cdot)}\exp\left(-\frac{1}{2}\log\left(f_1(x,y)/f_2(x,y)\right)\right)
\end{align*}
where the second inequality holds because $1-x\le-\log x$.
\end{proof}

\begin{lemma}[Adapted version of Theorem 21~\citep{agarwal2020flambe}]\label{lem:bound-tv}
    Fix $\delta\in(0,1)$. Let $N_{[]}(\epsilon,\+F,\|\cdot\|_\infty)$ denote the $\epsilon$-bracketing number of $\+F$ w.r.t. $\|\cdot\|_\infty$. Then for any estimator $\^f$ that depends on $D$, with probability at least $1-\delta$, we have
    \begin{align*}
        \sum_{i=1}^n\E_{x\sim\+D_i}&d_{tv}^2\left(\^f(x,\cdot), f^\star(x,\cdot)\right)\le\\
        &\frac{3n\epsilon^2|\+Y|^2}{2}+
        2n\epsilon|\+Y|+\big(4+2\epsilon|\+Y|\big)\left(\frac{1}{2}\sum_{i=1}^n \log\big(f^\star(x_i,y_i)/\^f(x_i,y_i)\big)+\log N_{[]}(\epsilon,\+F,\|\cdot\|_\infty) + \log(1/\delta)\right)
    \end{align*}
    where $|\+Y|$ denotes $\int_\+Y \d y$.
\end{lemma}
\begin{proof}[Proof of Lemma \ref{lem:bound-tv}]
    We take an $\epsilon$-bracket of $\+F$, $\{[l_i,u_i]:i=1,2,\dots\}$, and denote $\w~{\+F}=\{u_i:i=1,2,\dots\}$. Pick $\~f\in\w~{\+F}$ satisfying $\^f\leq\~f$, so $\~f$ also depends on $D$. Applying Lemma 24 of \cite{agarwal2020flambe} to function class $\w~{\+F}$ and estimator $\~f$ and using Chernoff method, we have
    \begin{equation}\label{eq:lem24}
        \underbrace{-\log\E_{D'}\exp(L(\~f(D),D'))}_{\rm(i)}
        \le
        \underbrace{-L(\~f(D),D)+\log N_{[]}(\epsilon,\+F,\|\cdot\|_\infty) + \log(1/\delta)}_{\rm(ii)}.
    \end{equation}
     holds with probability at least $1-\delta$. We set $L(f,D)=\sum_{i=1}^n -\nicefrac{1}{2} \log(f^\star(x_i,y_i)/f(x_i,y_i))$. Then the right hand side of \eqref{eq:lem24} is
     \begin{align*}
         {\rm{(ii)}}=&\frac{1}{2}\sum_{i=1}^n \log(f^\star(x_i,y_i)/\~f(x_i,y_i))+\log N_{[]}(\epsilon,\+F,\|\cdot\|_\infty) + \log(1/\delta)\\
         \le&\frac{1}{2}\sum_{i=1}^n \log(f^\star(x_i,y_i)/\^f(x_i,y_i))+\log N_{[]}(\epsilon,\+F,\|\cdot\|_\infty) + \log(1/\delta).
     \end{align*}
     On the other hand, by the definition of total variation distance and the fact that $a^2\le2b^2+2c^2$ whenever $0\le a\le b+c$, we have
     \begin{align*}
     	&\sum_{i=1}^n\E_{x\sim\+D_i}d_{tv}^2\left(\^f(x,\cdot), f^\star(x,\cdot)\right)
     	=\frac{1}{4}\sum_{i=1}^n\E_{x\sim\+D_i}\left(\int_{\+Y}\left|\^f(x,y)-f^\star(x,y)\right|\d y\right)^2\\
     	\le&
     	\frac{1}{2}\underbrace{\sum_{i=1}^n\E_{x\sim\+D_i}\left(\int_{\+Y}\left|\^f(x,y)-\~f(x,y)\right|\d y\right)^2}_{\rm{(iii)}}
     	+
     	\frac{1}{2}\underbrace{\sum_{i=1}^n\E_{x\sim\+D_i}\left(\int_{\+Y}\left|\~f(x,y)-f^\star(x,y)\right|\d y\right)^2}_{\rm{(iv)}}.
     \end{align*}
     
     For $\rm{(iii)}$, by the definition of $\~f$, we have ${\rm{(iii)}}\leq n\epsilon^2|\+Y|^2$. For $\rm{(iv)}$, we apply Lemma~\ref{lem:convert-1norm} with $f_1=f^\star$ and $f_2=\~f$ (thus $s=1+\epsilon|\+Y|$) and get    
     \begin{align*}
        {\rm(iv)}
        =&2n\epsilon|\+Y|(2+\epsilon|\+Y|)-\sum_{i=1}^n(8+4\epsilon|\+Y|)\left(\log\E_{x,y\sim f^\star(x,\cdot)}\exp\left(-\frac{1}{2}\log\left(f^\star(x,y)/\~f(x,y)\right)\right)\right)\\
        =&2n\epsilon|\+Y|(2+\epsilon|\+Y|)-\sum_{i=1}^n(8+4\epsilon|\+Y|)\left(\log\E_{x,y\sim\+D_i}\exp\left(-\frac{1}{2}\log\left(f^\star(x,y)/\~f(x,y)\right)\right)\right)\\
        =&2n\epsilon|\+Y|(2+\epsilon|\+Y|)-(8+4\epsilon|\+Y|)\log\E_{x,y\sim\+D'}\left[\exp\left(\sum_{i=1}^n-\frac{1}{2}\log\left(f^\star(x,y)/\~f(x,y)\right)\right)\middle|D\right]\\
        =&4n\epsilon|\+Y|+2n\epsilon^2|\+Y|^2+(8+4\epsilon|\+Y|)\cdot{\rm{(i)}}.
    \end{align*}
    By plugging $\rm{(iii)}$ and $\rm{(iv)}$ back we get
    \begin{align*}
    	\sum_{i=1}^n\E_{x\sim\+D_i}d_{tv}^2\left(\^f(x,\cdot), f^\star(x,\cdot)\right)
    	\leq2n\epsilon|\+Y|+\frac{3}{2}n\epsilon^2|\+Y|^2+(4+2\epsilon|\+Y|)\cdot{\rm{(i)}}.
    \end{align*}
    Notice that $\rm{(i)}\le\rm{(ii)}$, so we complete the proof by plugging $\rm(ii)$ into the above.
\end{proof}

\begin{lemma}\label{lem:mle-error}
	Fixed $\delta\in(0,1)$. Let $\^f$ denote the maximum likelihood estimator,
	\begin{equation*}
		\^f=\argmax_{f\in\+F}\sum_{i=1}^n\log f(x_i,y_i).
	\end{equation*}
	Then according to different assumptions on the size of $\+F$, we have the following two conclusions:
	\begin{enumerate}
		\item[(1)] If $|\+F|<\infty$, we have
		\begin{equation}\label{eq:lem-mle-error-finite}
			\sum_{i=1}^n\E_{x\sim\+D_i}d_{tv}^2\left(\^f(x,\cdot),f^\star(x,\cdot)\right)\le 4\log|\+F|/\delta
		\end{equation} with probability at least $1-\delta$.
		\item[(2)] For general $\+F$, we have
		\begin{equation}\label{eq:lem-mle-error-general}
			\sum_{i=1}^n\E_{x\sim\+D_i}d_{tv}^2\left(\^f(x,\cdot),f^\star(x,\cdot)\right)\le 10\log N_{[]}\left((n|\+Y|)^{-1},\+F,\|\cdot\|_\infty\right)/\delta
		\end{equation} with probability at least $1-\delta$.
	\end{enumerate}
\end{lemma}
\begin{proof}[Proof of Lemma~\ref{lem:mle-error}]
	By Lemma~\ref{lem:bound-tv}, we have
	\begin{equation}\label{eq:lem-mle-error-0}
    \begin{aligned}
        \sum_{i=1}^n\E_{x\sim\+D_i}&d_{tv}^2\left(\^f(x,\cdot), f^\star(x,\cdot)\right)\le\\
        &\frac{3n\epsilon^2|\+Y|^2}{2}+
        2n\epsilon|\+Y|+\big(4+2\epsilon|\+Y|\big)\left(\frac{1}{2}\underbrace{\sum_{i=1}^n \log\big(f^\star(x_i,y_i)/\^f(x_i,y_i)\big)}_{\rm(\diamond)}+\log N_{[]}(\epsilon,\+F,\|\cdot\|_\infty) + \log(1/\delta)\right)
    \end{aligned}
	\end{equation}
	with probability at least $1-\delta$. Since $\^f$ is the maximum likelihood estimator and there exists $f^\star_i$ that agrees with $f^\star$ on the support of $\+D_i$, we have 
	\begin{align*}
		\log\big(f^\star(x_i,y_i)/\^f(x_i,y_i)\big)
		=
		\log\big(f^\star_i(x_i,y_i)/\^f(x_i,y_i)\big)
		\leq 0
	\end{align*} 
	and thus${\rm(\diamond)}\le 0$. When $|\+F|<\infty$, we can set $\epsilon=0$, and then \eqref{eq:lem-mle-error-0} exactly becomes \eqref{eq:lem-mle-error-finite}. For general $\+F$, we set $\epsilon=(n|\+Y|)^{-1}$ and then get
    \begin{align*}
        &\sum_{i=1}^n\E_{x\sim\+D_i}d_{tv}^2\left(\^f(x,\cdot), f^\star(x,\cdot)\right)
        \le
        \frac{3}{2n}+
        2+\left(4+\frac{2}{n}\right)\log N_{[]}((n|\+Y|)^{-1},\+F,\|\cdot\|_\infty)/\delta)\\
        \le&
        4+6\log N_{[]}((n|\+Y|)^{-1},\+F,\|\cdot\|_\infty)/\delta
        \le
        10\log N_{[]}((n|\+Y|)^{-1},\+F,\|\cdot\|_\infty)/\delta,
    \end{align*}
    which is exactly \eqref{eq:lem-mle-error-general}.
\end{proof}

\subsection{Total Variation Distance and Wasserstein Distance}\label{app:metric}

The following lemma states that the total variation distance is equal to the optimal coupling in a sense. The proof can be found in \citet{levin2017markov} (Proposition 4.7).

\begin{lemma}\label{lem:tv-coupling}
	Let $f_1$ and $f_2$ be two probability distributions on $\+X$. Then
	\begin{align*}
		d_{tv}(f_1,f_2)
		=\inf_{c\in\+C} \Pr_{x,y\sim c}(x\neq y)
	\end{align*}
	where $\+C$ is the set of all couplings of $f_1$ and $f_2$.
\end{lemma}

The following lemma shows the dual representation of the Wasserstein distance. The proof can be found in \citet{villani2021topics} (Theorem 1.3) and \citet{villani2009optimal} (Theorem 5.10). 
\begin{lemma}[Kantorovich duality]\label{lem:kan-dual}
Let $f_1,f_2\in\Delta(\+X)$ where $\+X$ is a Polish space (e.g., Euclidean space). It can be shown that, for any $1\leq p<\infty$,
\begin{align*}
	d_{w,p}^p(f_1,f_2)=\sup_{\psi,\phi}
	\int \psi(x) f_1(x) \d x
	-
	\int \phi(x) f_2(x) \d x
	\text{\quad s.t.\quad}
	\psi(x)-\phi(y)\leq \|x-y\|^p, \quad\forall x,y\in\+X.
\end{align*}
\end{lemma}

\begin{lemma}\label{lem:tv-wass-general}
	Let $f_1$ and $f_2$ be two distributions on a bounded set $\+X$. Then
	\begin{equation*}
		d_{w,p}^p(f_1,f_2)\le \text{\rm diam}^p(\+X)\cdot d_{tv}(f_1,f_2)
	\end{equation*}
	where $\text{\rm diam($\+X$)}=\sup_{x,y\in\+X}\|x-y\|$ is the diameter of $\+X$.
\end{lemma}
\begin{proof}
	By definition, we have 
	\begin{align*}
		&d_{w,p}^p(f_1,f_2)
		=\inf_{c\in\+C} \E_{x,y\sim c} \|x-y\|^p
		=\inf_{c\in\+C} \E_{x,y\sim c} \big[\indic[x\neq y]\cdot\|x-y\|^p\big]\\
		\leq&\text{\rm diam}^p(\+X)\cdot \inf_{c\in\+C} \E_{x,y\sim c} \indic[x\neq y]
		=\text{\rm diam}^p(\+X)\cdot d_{tv}(f_1,f_2)
	\end{align*}
	where by $\+C$ we denote the set of all couplings of $f_1$ and $f_2$, and the last equality holds because of \cref{lem:tv-coupling}.
\end{proof}

\begin{corollary}\label{lem:tv-wass}
	Let $f_1$ and $f_2$ be two distributions on $[0, m]^d$. Then
	\begin{equation*}
		d_{w,p}^p(f_1,f_2)\le \left(m\sqrt{d}\right)^p\cdot d_{tv}(f_1,f_2).
	\end{equation*}
\end{corollary}

Since the total variation distance is at most one, we have the following.
\begin{corollary}\label{lem:wass-diam}
	Let $f_1$ and $f_2$ be two distributions on $[0, m]^d$. Then
	\begin{equation*}
		d_{w,p}^p(f_1,f_2)\le \left(m\sqrt{d}\right)^p.
	\end{equation*}
\end{corollary}

\section{Missing Proofs in Section~\ref{sec:analysis}}\label{app:proofs}

\subsection{Proof of Theorem~\ref{thm:bell-error-to-final}}\label{app:pf-bell-error-to-final}

\begin{proof}

Note that for all $h\in[H]$, we have
\begin{align*}
    &\E_{x,a\sim d^\pi_h}\ 
    d_{tv}\left([\+T^\pi\^f_{h+1}](x,a), [\+T^\pi Z^\pi_{h+1}](x,a)\right)\\
    =&\frac{1}{2}
    \E_{x,a\sim d^\pi_h}\sup_{g:\|g\|_\infty\le1}
    \left|
    \E_{\substack{x'\sim P(x,a)\\a'\sim\pi(x')\\r\sim r(x,a)}}
    \left(
    \E_{y\sim\^f_{h+1}(\cdot|x',a')}
    g\big(r+y\big)
    -
    \E_{y\sim Z^\pi_{h+1}(\cdot|x',a')}
    g\big(r+y\big)
    \right)
    \right|\\
    \le&
    \frac{1}{2}
    \E_{\substack{x,a\sim d^\pi_h\\x'\sim P(x,a)\\a'\sim\pi(\cdot|x)\\r\sim r(x,a)}}
    \sup_{g:\|g\|_\infty\le1}
    \left|
    \E_{y\sim\^f_{h+1}(\cdot|x',a')}
    g\big(r+y\big)
    -
    \E_{y\sim Z^\pi_{h+1}(\cdot|x',a')}
    g\big(r+y\big)
    \right|\\
    =&
    \frac{1}{2}
    \E_{\substack{x,a\sim d^\pi_h\\x'\sim P(x,a)\\a'\sim\pi(\cdot|x)}}
    \sup_{g:\|g\|_\infty\le1}
    \left|
    \E_{y\sim\^f_{h+1}(\cdot|x',a')}
    g(y)
    -
    \E_{y\sim Z^\pi_{h+1}(\cdot|x',a')}
    g(y)
    \right|\\
    =&\E_{\substack{x',a'\sim d^\pi_{h+1}}}
    d_{tv}\left(\^f_{h+1}(x',a'),Z^\pi_{h+1}(x',a')\right).
\end{align*} 

Here the inequality holds for Jensen's inequality. The second equality holds since the randomness of $r$ lies outside the supremum, so we can consider $r$ as a constant within the supremum, allowing us to set $\~g(y)=g(r+y)$ for which we have $\|\~g\|_\infty\le1$ thus removing the additive term $r$. Hence, by triangle inequality, we have
\begin{align*}
    &\E_{x,a\sim d^\pi_h}\ d_{tv}\left(\^f_h(x,a), Z^\pi_h(x,a)\right)
    =\E_{x,a\sim d^\pi_h}\ d_{tv}\left(\^f_h(x,a), [\+T^\pi Z^\pi_{h+1}](x,a)\right)\\
    \le&
    \underbrace{
    \E_{x,a\sim d^\pi_h}\ d_{tv}\left(\^f_h(x,a), [\+T^\pi\^f_{h+1}](x,a)\right)
    }_{\rm(i)}
    +
    \underbrace{
    \E_{x,a\sim d^\pi_h}\ d_{tv}\left([\+T^\pi\^f_{h+1}](x,a), [\+T^\pi Z^\pi_{h+1}](x,a)\right)\
    }_{\rm(ii)}.
\end{align*}
By Assumption~\ref{asm:cover} and Jensen's inequality, we have ${\rm(i)}\le \sqrt{C}\zeta_h$ because
\begin{align*}
     & \E_{x,a\sim d^\pi_h}\ d_{tv}\left(\^f_h(x,a), [\+T^\pi\^f_{h+1}](x,a)\right) \\
     &\leq \left \{ \E_{x,a\sim d^\pi_h}\ d^2_{tv}\left(\^f_h(x,a), [\+T^\pi\^f_{h+1}](x,a)\right) \right \}^{1/2}\leq \sqrt{C}\zeta_h. 
\end{align*}

And by the above derivation we have ${\rm(ii)}\le\E_{x,a\sim d^\pi_{h+1}}d_{tv}\left(\^f_{h+1}(x,a),Z^\pi_{h+1}(x,a)\right)$. Hence, 
\begin{equation*}
    \E_{x,a\sim d^\pi_h}\ d_{tv}\left(\^f_h(x,a), Z^\pi_h(x,a)\right)
    \le
    \sqrt{C}\zeta_h
    +
    \E_{x,a\sim d^\pi_{h+1}}d_{tv}\left(\^f_{h+1}(x,a),Z^\pi_{h+1}(x,a)\right).
\end{equation*}

Summing over $h=1,\dots,H$ on both sides, we get
\begin{equation*}
    \E_{x,a\sim d^\pi_1}\ d_{tv}\left(\^f_1(x,a), Z^\pi_1(x,a)\right)
    \le
    \sqrt{C}\sum_{h=1}^H\zeta_h
    +
    \E_{x,a\sim d^\pi_{H+1}}d_{tv}\left(\^f_{H+1}(x,a),Z^\pi_{H+1}(x,a)\right)=\sqrt{C}\sum_{h=1}^H\zeta_h.
\end{equation*}
where the equality holds since $\^f_{H+1}=Z^\pi_{H+1}=0$ by definition. Now we complete the proof by noticing the following
\begin{align*}
    &d_{tv}\left(\^f,Z^\pi\right)
    =
    \frac{1}{2}\sup_{g:\|g\|_\infty\le1}
    \left|
    \E_{x,a\sim d^\pi_1}
    \left(
    \E_{y\sim \^f_1(\cdot|x,a)}
    g(y)
    -
    \E_{y\sim Z^\pi(\cdot|x,a)}
    g(y)
    \right)
    \right|\\
    \le&
    \frac{1}{2}
    \E_{x,a\sim d^\pi_1}
    \sup_{g:\|g\|_\infty\le1}
    \left|
    \E_{y\sim \^f_1(\cdot|x,a)}
    g(y)
    -
    \E_{y\sim Z^\pi(\cdot|x,a)}
    g(y)
    \right|
    =
    \E_{x,a\sim d^\pi_1}\ d_{tv}\left(\^f_1(x,a), Z^\pi_1(x,a)\right).
\end{align*}
\end{proof}

\subsection{Proof of Lemma~\ref{lem:mle}}\label{app:pf-lem-mle}

\begin{proof}
Observing Algorithm~\ref{alg:mle-ope}, when $h=H$, we are estimating the conditional distribution $Z^\pi_H$ via MLE. Under Assumption~\ref{asm:bell-comp} which implies that there exists a function $g\in\+F_H$ that agrees with $Z^\pi_H$ on the support of $\rho$, we can apply Lemma~\ref{lem:mle-error}, which leads to
\begin{equation*}
	\E_{x,a\sim\rho}d_{tv}^2\left(\^f_H(x,a),Z^\pi_H(x,a)\right)\le \frac{4H}{n}\log(|\+F_H|/\delta)
\end{equation*} 
with probability at least $1-\delta$. When $h<H$, we are estimating the conditional distribution $\+T^\pi\^f_{h+1}$ via MLE. Also note that thanks to the random data split, we have $\hat f_{h+1}$ being independent of the dataset $\mathcal{D}_h$ ($\hat f_{h+1}$ only depends on datasets $\mathcal{D}_{h+1},\dots \mathcal{D}_H$).
Therefore, under Assumption~\ref{asm:bell-comp} which implies that there exists a function $g\in\+F_h$ that agrees with $\+T^\pi\^f_{h+1}$ on the support of $\rho$, we can apply Lemma~\ref{lem:mle-error}, which leads to
\begin{equation*}
	\E_{x,a\sim\rho}d_{tv}^2\left(\^f_h(x,a),[\+T^\pi \^f_{h+1}](x,a)\right)\le \frac{4H}{n}\log(|\+F_H|/\delta)
\end{equation*}
with probability at least $1-\delta$.
We complete the proof by taking the union bound for $h\in[H]$.

\end{proof}

\subsection{Proof of Lemma~\ref{lem:mle-general}}\label{app:pf-lem-general}

\begin{proof}
The proof is similar to Lemma~\ref{lem:mle}. Observing Algorithm~\ref{alg:mle-ope}, when $h=H$, we are basically estimating the conditional distribution $Z^\pi_H$ via MLE. Hence, under Assumption~\ref{asm:bell-comp} which implies that there exists a function $g\in\+F_H$ that agrees with $Z^\pi_H$ on the support of $\rho$, we can apply Lemma~\ref{lem:mle-error}, which leads to

\begin{equation*}
	\E_{x,a\sim\rho}d_{tv}^2\left(\^f_H(x,a),Z^\pi_H(x,a)\right)\le \frac{10H}{n}\log\left(N_{[]}\big((nH^d)^{-1},\+F_H,\|\cdot\|_\infty\big)/\delta\right)
\end{equation*}
with probability at least $1-\delta$. When $h<H$, we are estimating the conditional distribution $\+T^\pi\^f_{h+1}$ via MLE. Therefore, under Assumption~\ref{asm:bell-comp} which implies that there exists a function $g\in\+F_h$ that agrees with $\+T^\pi\^f_{h+1}$ on the support of $\rho$, we can apply Lemma~\ref{lem:mle-error}, which leads to\begin{equation*}
	\E_{x,a\sim\rho}d_{tv}^2\left(\^f_h(x,a),[\+T^\pi \^f_{h+1}](x,a)\right)\le \frac{10H}{n}\log\left(N_{[]}\big((nH^d)^{-1},\+F_h,\|\cdot\|_\infty\big)/\delta\right)
\end{equation*}
with probability at least $1-\delta$.
We complete the proof by taking the union bound for $h\in[H]$.
\end{proof}

\subsection{Proof of Lemma~\ref{lem:contractive}}\label{app:pf-lem-contractive}
\begin{proof}

First, it deserves to verify that the ``metric'' $(\E_{x,a\sim d^\pi}d^{2p}_{w,p})^{1/(2p)}$ we are using satisfies the triangle inequality and is thus indeed a metric.  To this end, we note that, for any three densities $f_1,f_2,f_3:\+X\times\+A\mapsto\Delta([0,(1-\gamma)^{-1}]^d)$, the following holds since $d_{w,p}$ is a metric,
\begin{align*}
	\bigg(\E_{x,a\sim d^\pi} d^{2p}_{w,p}(f_1(x,a), f_2(x,a))\bigg)^{\frac{1}{2p}}
		\leq
	\bigg(\E_{x,a\sim d^\pi} \Big(d_{w,p}(f_1(x,a), f_3(x,a))+d_{w,p}(f_3(x,a), f_2(x,a))\Big)^{2p}\bigg)^{\frac{1}{2p}}.
\end{align*}
Then by Minkowski inequality, the above
\begin{align*}
\leq
	\left(\E_{x,a\sim d^\pi} d^{2p}_{w,p}(f_1(x,a), f_3(x,a))\right)^{\frac{1}{2p}}
	+
	\left(\E_{x,a\sim d^\pi} d^{2p}_{w,p}(f_3(x,a), f_2(x,a))\right)^{\frac{1}{2p}},
\end{align*}
for which we conclude triangle inequality for $(\E_{x,a\sim d^\pi}d^{2p}_{w,p})^{1/(2p)}$. Since other axioms of metrics are trivial to verify, we conclude that it is indeed a metric. Hence, we can safely proceed.

To establish the contractive property, we start with the following lemma, which shows that the distributional Bellman operator is roughly ``$\gamma$-contractive'' in a sense but with distribution shifts.
\begin{lemma}\label{lem:wp-contract}
For any $f,f'\in\+F$, $x\in\+X$ and $a\in\+A$, we have
\begin{align*}
	d_{w,p}^p\left([\+T^\pi f](x,a),[\+T^\pi f'](x,a)\right)
	\leq
	\E_{x'\sim P(x,a),a'\sim\pi(x')} \gamma^p d^p_{w,p}\left(f(x',a'),f'(x',a')\right).
\end{align*}
\end{lemma}	
\begin{proof}[Proof of \cref{lem:wp-contract}]
By the dual form of Wasserstein distance (\cref{lem:kan-dual}), we have 
\begin{align*}
	&d_{w,p}^p\left([\+T^\pi f](x,a),[\+T^\pi f'](x,a)\right)
	=\sup_{(\psi,\phi)\in\Gamma}\E_{z\sim[\+T^\pi f](x,a)}\psi(z)-\E_{z\sim[\+T^\pi f'](x,a)}\phi(z)\\
	=&\sup_{(\psi,\phi)\in\Gamma}\E_{\substack{x'\sim P(x,a)\\a'\sim\pi(x')\\r\sim r(x,a)}}
	\left(\E_{y\sim f(x',a')}\psi(r+\gamma y)-\E_{y\sim f'(x',a')}\phi(r+\gamma y)\right)\\
	\leq&\E_{\substack{x'\sim P(x,a)\\a'\sim\pi(x')\\r\sim r(x,a)}}
	\underbrace{\sup_{(\psi,\phi)\in\Gamma}
	\left(\E_{y\sim f(x',a')}\psi(r+\gamma y)-\E_{y\sim f'(x',a')}\phi(r+\gamma y)\right)}_{(*)}\numberthis\label{eq:sup-contract}
\end{align*}
where $\Gamma=\{(\psi,\phi):\psi(x)-\phi(y)\leq \|x-y\|^p\}$.  The second equality holds by the definition of Bellman operator. 

Regarding $\rm(*)$, for any $(\psi,\phi)\in\Gamma$, we define $\w~\psi(y)=\psi(r+\gamma y)/\gamma^p$ and $\w~\phi(y)=\phi(r+\gamma y)/\gamma^p$. Then, we have
\begin{align*}
	(*)
	=
	\gamma^p
	\sup_{(\psi,\phi)\in\Gamma}
	\left(\E_{y\sim f(x',a')}\w~\psi(y)-\E_{y\sim f'(x',a')}\w~\phi(y)\right).
\end{align*}

We note that, for any $x,y$,
\begin{align*}
	\w~\psi(x)-\w~\phi(y)
	=
	\frac{\psi(r+\gamma x)-\phi(r+\gamma y)}{\gamma^p}
	\leq\frac{\|(r+\gamma x)-(r+\gamma y)\|^p}{\gamma^p}
	=\|x-y\|^p.
\end{align*}

Here the inequality holds since $(\psi,\phi)\in\Gamma$. Hence, $(\w~\psi,\w~\phi)\in\Gamma$ as well. In other words, for any given $\psi$ and $\phi$, their correspondences $\w~\psi$ and $\w~\phi$ are also in $\Gamma$. Thus we can take the supremum directly over the latter, which leads to
\begin{align*}
	(*)
	\leq
	\gamma^p
	\sup_{(\w~\psi,\w~\phi)\in\Gamma}
	\left(\E_{y\sim f(x',a')}\w~\psi(y)-\E_{y\sim f'(x',a')}\w~\phi(y)\right)
	=
	\gamma^p
	d^p_{w,p}\big(f(x',a'),f'(x',a')\big)
\end{align*}
where the equality holds due to the dual form of Wasserstein distance (\cref{lem:kan-dual}) again. Then we plug the above into \eqref{eq:sup-contract} and get
\begin{align*}
	&d_{w,p}^p\left([\+T^\pi f](x,a),[\+T^\pi f'](x,a)\right)
\leq\E_{\substack{x'\sim P(x,a),a'\sim\pi(x')}}
	\gamma^p
	d^p_{w,p}\big(f(x',a'),f'(x',a')\big).
\end{align*}
where we have removed the randomness of $r\sim r(x,a)$ originally appeared in \eqref{eq:sup-contract} since the the term inside the expectation is now completely independent of $r$.
\end{proof}

By \cref{lem:wp-contract}, we have
\begin{align*}
	&\left(\E_{x,a\sim d^\pi}d^{2p}_{w,p}\left([\+T^\pi f](x,a),[\+T^\pi f'](x,a)\right)\right)^{\frac{1}{2p}}
	=
	\left(\E_{x,a\sim d^\pi}\Big(
	d^p_{w,p}\left([\+T^\pi f](x,a),[\+T^\pi f'](x,a)\right)\Big)^2\right)^{\frac{1}{2p}}\\
	\leq&
	\gamma\cdot
	\left(\E_{x,a\sim d^\pi}\left(
	\E_{x'\sim P(x,a),a'\sim \pi(x')}
	d^p_{w,p}\left(f(x',a'),f'(x',a')\right)\right)^2\right)^{\frac{1}{2p}}\\
	\leq&
	\gamma\cdot
	\left(
	\underbrace{\E_{\substack{x,a\sim d^\pi\\x'\sim P(x,a),a'\sim \pi(x')}}
	d^{2p}_{w,p}\left(f(x',a'),f'(x',a')\right)}_{(\dagger)}\right)^{\frac{1}{2p}}
\end{align*}
where the last inequality holds because of Jensen's inequality. Since $d^\pi(x,a)=\gamma\E_{\~x,\~a\sim d^\pi} P(x|\~x,\~a)\pi(a|x)+(1-\gamma) \mu(x)\pi(x|a)$, we have $\E_{\~x,\~a\sim d^\pi} P(x|\~x,\~a)\pi(a|x)\leq \gamma^{-1}d^\pi(x,a)$. Therefore,
\begin{align*}
	(\dagger)\leq \gamma^{-1} \E_{x,a\sim d^\pi}d^{2p}_{w,p}\left(f(x,a), f'(x,a)\right).
\end{align*}
Hence, we conclude that
\begin{align*}
	&\left(\E_{x,a\sim d^\pi}d^{2p}_{w,p}\left([\+T^\pi f](x,a),[\+T^\pi f'](x,a)\right)\right)^{\frac{1}{2p}}\\
	\leq&
	\gamma\cdot
	\left(\gamma^{-1}\E_{x,a\sim d^\pi} d^{2p}_{w,p}\left(f(x,a), f'(x,a)\right)\right)^{\frac{1}{2p}}\\
	=&
	\gamma^{1-\frac{1}{2p}}\cdot
	\left(\E_{x,a\sim d^\pi} d^{2p}_{w,p}\left(f(x,a), f'(x,a)\right)\right)^{\frac{1}{2p}}.
\end{align*}
\end{proof}

\subsection{Proof of Theorem~\ref{thm:bell-error-to-final-inf}}
\begin{proof}

We will prove the following theorem which is more general. 

\begin{theorem}\label{thm:refined}
Under \cref{asm:cover-inf}, suppose we have a sequence of functions $\^f_1,\dots,\^f_T:\+X\times\+A\mapsto\Delta([0,(1-\gamma)^{-1}]^d)$ and a sequence of values $\zeta_1,\dots,\zeta_T\in\=R$ such that
\begin{equation*}
 \textstyle \bigg(\E_{x,a\sim\rho}\ d_{w,p}^{2p}\left(\^f_t(x,a), [\+T^\pi \^f_{t-1}](x,a)\right)\bigg)^{\frac{1}{2p}}\le \zeta_t
\end{equation*}
holds for all $t\in[T]$. Let our estimator $\^f\coloneqq\E_{x\sim\mu,a\sim\pi(x)}\^f_T(x,a)$. Then we have, for all $p\ge1$,
\begin{equation}\label{eq:geo-zeta_t-ap}
	d_{w,p}\left(\^f, Z^\pi\right)
	\leq
	\left(\frac{C}{1-\gamma}\right)^{\frac{1}{2p}}
	\sum_{t=1}^T \gamma^{(T-t)\left(1-\frac{1}{2p}\right)}\cdot\zeta_t
	+\frac{\sqrt{d}\cdot\gamma^{T\left(1-\frac{1}{2p}\right)}}{(1-\gamma)^{1+\frac{1}{2p}}}.
\end{equation}
\end{theorem}

\begin{proof}[Proof of \cref{thm:refined}]
Recall that we defined the conditional distribuions $\-{Z}^\pi(x,a)\in\Delta([0,(1-\gamma)^{-1}]^d)$ which is the distribution of the return under policy $\pi$ starting with state action $(x,a)$. 
It is easy to see that $Z^\pi =\mathbb{E}_{x\sim \mu, a\sim \pi(x)} \left[ \-Z^\pi(x,a)\right]$. We start with the following.
\begin{align*}
	&\left(\E_{x,a\sim d^\pi} d^{2p}_{w,p}\left(\^f_t(x,a), \-{Z}^\pi(x,a)\right)\right)^{\frac{1}{2p}}\\
	\leq&\left(\E_{x,a\sim d^\pi} d^{2p}_{w,p}\left(\^f_t(x,a), [\+T^\pi\^f_{t-1}](x,a)\right)\right)^{\frac{1}{2p}}
	+\left(\E_{x,a\sim d^\pi} d^{2p}_{w,p}\left([\+T^\pi\^f_{t-1}](x,a), \-{Z}^\pi(x,a)\right)\right)^{\frac{1}{2p}}\\
	\leq&\ C^{\frac{1}{2p}}\left(\E_{x,a\sim \rho} d^{2p}_{w,p}\left(\^f_t(x,a), [\+T^\pi\^f_{t-1}](x,a)\right)\right)^{\frac{1}{2p}}
	+\left(\E_{x,a\sim d^\pi} d^{2p}_{w,p}\left([\+T^\pi\^f_{t-1}](x,a), [\+T^\pi \-{Z}^\pi](x,a)\right)\right)^{\frac{1}{2p}}\\
	\leq&C^{\frac{1}{2p}}\zeta_t+\gamma^{1-\frac{1}{2p}}\left(\E_{x,a\sim d^\pi}d^{2p}_{w,p}\left(\^f_{t-1}(x,a), \-{Z}^\pi(x,a)\right)\right)^{\frac{1}{2p}}
\end{align*}
where the first inequality is due to triangle inequality (proved in \cref{app:pf-lem-contractive}), the second inequality holds because of the coverage assumption (\cref{asm:cover-inf}), and the last inequality holds due to the contractive property of the distributional Bellman operator (\cref{lem:contractive}). Unrolling the recursion of $t$, we arrive at
\begin{align*}
	&\left(\E_{x,a\sim d^\pi} d^{2p}_{w,p}\left(\^f_T(x,a), \-{Z}^\pi(x,a)\right)\right)^{\frac{1}{2p}}\\
	\leq& 
	\sum_{t=1}^T \gamma^{(T-t)\left(1-\frac{1}{2p}\right)} C^{\frac{1}{2p}}\zeta_t
	+\gamma^{T\left(1-\frac{1}{2p}\right)}\left(\E_{x,a\sim d^\pi}d^{2p}_{w,p}\left(\^f_0(x,a), \-{Z}^\pi(x,a)\right)\right)^{\frac{1}{2p}}\\
	\leq& 
	\sum_{t=1}^T \gamma^{(T-t)\left(1-\frac{1}{2p}\right)} C^{\frac{1}{2p}}\zeta_t
	+\gamma^{T\left(1-\frac{1}{2p}\right)}\cdot\frac{\sqrt{d}}{1-\gamma}\numberthis\label{eq:step1}
\end{align*}
where the last inequality is due to \cref{lem:wass-diam} which shows that  
$$d_{w,p}(\^f_0(x,a), \-{Z}^\pi(x,a))\leq \text{\rm diam}\big([0,(1-\gamma)^{-1}]^d\big)\leq \frac{\sqrt{d}}{(1-\gamma)}.$$

Since $d^\pi(x,a)=\gamma\E_{\~x,\~a\sim d^\pi} P(x|\~x,\~a)\pi(a|x)+(1-\gamma) \mu(x)\pi(x\given a)$, we have $\mu(x)\pi(x\given a)\leq (1-\gamma)^{-1} d^\pi(x,a)$ and thus
\begin{align*}
	&\left(\E_{x\sim\mu,a\sim\pi(x)} d^{2p}_{w,p}\left(\^f_T(x,a), \-{Z}^\pi(x,a)\right)\right)^{\frac{1}{2p}}
	\leq
	\left((1-\gamma)^{-1}\E_{x,a\sim d^\pi} d^{2p}_{w,p}\left(\^f_T(x,a), \-{Z}^\pi(x,a)\right)\right)^{\frac{1}{2p}}\\
	=&(1-\gamma)^{-\frac{1}{2p}}\left(\E_{x,a\sim d^\pi} d^{2p}_{w,p}\left(\^f_T(x,a), \-{Z}^\pi(x,a)\right)\right)^{\frac{1}{2p}}
	\leq
	(1-\gamma)^{-\frac{1}{2p}}\left(\sum_{t=1}^T \gamma^{(T-t)\left(1-\frac{1}{2p}\right)} C^{\frac{1}{2p}}\zeta_t
	+\gamma^{T\left(1-\frac{1}{2p}\right)}\cdot\frac{\sqrt{d}}{1-\gamma}\right)\numberthis\label{eq:step2}
\end{align*}
where the last inequality is for \eqref{eq:step1}.

Applying the dual representation of Wasserstein distance (\cref{lem:kan-dual}) to $d_{w,p}^p\left(\^f, Z^\pi\right)$, we have
\begin{align*}
	&d_{w,p}^p\left(\^f, Z^\pi\right)
	=d_{w,p}^p\left(\E_{x\sim\mu,a\sim\pi(x)}\^f_T(x,a), \E_{x\sim\mu,a\sim\pi(x)}\-{Z}^\pi(x,a)\right)\\
	=&\sup_{\psi,\phi\in\Gamma}\E_{x\sim\mu,a\sim\pi(x)}\left(\E_{z\sim\^f_T(x,a)}\psi(z)-\E_{z\sim\-Z^\pi(x,a)}\phi(z)\right)\\
	\leq&\E_{x\sim\mu,a\sim\pi(x)}\sup_{\psi,\phi\in\Gamma}\left(\E_{z\sim\^f_T(x,a)}\psi(z)-\E_{z\sim\-Z^\pi(x,a)}\phi(z)\right)\\
	=&\E_{x\sim\mu,a\sim\pi(x)}d^p_{w,p}\left(\^f_T(x,a),\-Z^\pi(x,a)\right)\\
	\leq&\left(\E_{x\sim\mu,a\sim\pi(x)}d^{2p}_{w,p}\left(\^f_T(x,a),\-Z^\pi(x,a)\right)\right)^{\frac{1}{2}}.\numberthis\label{eq:step3}
\end{align*}
where $\Gamma=\{(\psi,\phi):\psi(x)-\phi(y)\leq \|x-y\|^p\}$. By chaining \eqref{eq:step2} and \eqref{eq:step3} we complete the proof.
\end{proof}

By assuming there exists a common upper bound $\zeta$ (i.e., $\zeta_t\leq\zeta$, $\forall t$), we can further simplify \eqref{eq:geo-zeta_t-ap} by noticing the following. First, since the sum of geometric series is bounded in the following sense
\begin{equation*}
\sum_{t=1}^T\gamma^{(T-t)\left(1-\frac{1}{2p}\right)}
	\leq
	\frac{1}{1-\gamma^{\left(1-\frac{1}{2p}\right)}},
\end{equation*}
we can get
\begin{equation*}
	d_{w,p}\left(\^f, Z^\pi\right)
	\leq
	\left(\frac{C}{1-\gamma}\right)^{\frac{1}{2p}}
	\frac{\zeta}{\left(1-\gamma^{1-\frac{1}{2p}}\right)}
	+\frac{\sqrt{d}\cdot\gamma^{T\left(1-\frac{1}{2p}\right)}}{(1-\gamma)^{1+\frac{1}{2p}}}.
\end{equation*}
Second, we note that the right-hand side above attains the maximum when $p=1$. Therefore
\begin{align*}
	d_{w,p}\left(\^f, Z^\pi\right)
	\leq&
	\left(\frac{C}{1-\gamma}\right)^{\frac{1}{2p}}\cdot
	\frac{\zeta}{1-\gamma^{\frac{1}{2}}}
	+\frac{\sqrt{d}\cdot\gamma^{\frac{T}{2}}}{(1-\gamma)^{\frac{3}{2}}}\\
	\leq&
	\frac{2C^{\frac{1}{2p}}}{(1-\gamma)^{\frac{3}{2}}}
	\cdot \zeta
	+\frac{\sqrt{d}\cdot\gamma^{\frac{T}{2}}}{(1-\gamma)^{\frac{3}{2}}}
\end{align*}
where the last inequality holds since $1-\gamma^{1/2}\geq(1-\gamma)/2$.
\end{proof}

\subsection{Proof of Lemma~\ref{lem:mle-inf}}\label{app:pf-lem-mle-inf}

\begin{proof}
We only show the proof for finite function class since the proof for infinite class is essentially the same. 

For Algorithm~\ref{alg:mle-ope-inf}, we are iteratively estimating the conditional distribution $\+T^\pi\^f_{t-1}$. Note that thanks to the random data split, we have $\^f_{t-1}$ being independent of the dataset $\mathcal{D}_t$ ($\hat f_{t-1}$ only depends on datasets $\mathcal{D}_1,\dots \mathcal{D}_{t-1}$).
Therefore, under Assumption~\ref{asm:bell-comp-inf} which implies that there exists a function $g\in\+F$ that agrees with $\+T^\pi\^f_{t-1}$ on the support of $\rho$, we can apply Lemma~\ref{lem:mle-error}, which leads to
\begin{equation*}
	\E_{x,a\sim\rho}d_{tv}^2\left(\^f_t(x,a),[\+T^\pi \^f_{t-1}](x,a)\right)\le \frac{4T}{n}\log(|\+F|T/\delta)
\end{equation*}
with probability at least $1-\delta$. Here we have taken the union bound for $t\in[T]$. For the result of Wasserstein distance, we apply \cref{lem:tv-wass} and get
	\begin{align*}
		&\E_{x,a\sim\rho}d_{w,p}^{2p}\left(\^f_t(x,a),[\+T^\pi \^f_{t-1}](x,a)\right)
		\leq\left(\frac{\sqrt{d}}{1-\gamma}\right)^{2p} \E_{x,a\sim\rho}d_{tv}^2\left(\^f_t(x,a),[\+T^\pi \^f_{t-1}](x,a)\right).
	\end{align*}
	
\end{proof}

\subsection{Proof of Corollary~\ref{cor:mle-main-inf}}
\begin{proof}
	We only prove for the finite function class ($|\+F|<\infty$) since the proof for the infinite function class is quite similar. 
	
	We start with \cref{thm:bell-error-to-final-inf}, plug in the result of \cref{lem:mle-inf}, and get
\begin{align*}
	d_{w,p}\left(\^f, Z^\pi\right)
	\leq&
	\frac{2C^{\frac{1}{2p}}}{(1-\gamma)^{\frac{3}{2}}}
	\cdot
	\frac{\sqrt{d}}{1-\gamma}
	\cdot
	\left(\frac{4T}{n}\log(|\+F|T/\delta)\right)^{\frac{1}{2p}}
	+\frac{\sqrt{d}\cdot\gamma^\frac{T}{2}}{(1-\gamma)^{\frac{3}{2}}}\\
	=&
	\frac{\sqrt{d}}{(1-\gamma)^{\frac{3}{2}}}
	\left(
	\frac{2C^{\frac{1}{2p}}}{1-\gamma}
	\cdot
	\left(\frac{4T}{n}\log(|\+F|T/\delta)\right)^{\frac{1}{2p}}
	+\gamma^{\frac{T}{2}}
	\right)\numberthis\label{eq:tmp1}
\end{align*}
We choose
\begin{align*}
T=\frac{\log\left(C^{\frac{1}{2p}}\cdot\iota^{\frac{1}{2p}}\cdot \left(1-\gamma\right)^{-1}\cdot n^{-\frac{1}{2p}}\right)}{\log\left(\gamma^{\frac{1}{2}}\right)}
\text{\quad where\quad}
\iota=\log(|\+F|/\delta),
\end{align*}
which leads to
\begin{align*}
\gamma^{\frac{T}{2}}=\frac{C^{\frac{1}{2p}}\cdot\iota^{\frac{1}{2p}}\cdot n^{-\frac{1}{2p}}}{1-\gamma}.
\end{align*}
Thus, the second additive term of \eqref{eq:tmp1} will be smaller than the first one. Hence, we conclude that
\begin{equation*}
	d_{w,p}\left(\^f, Z^\pi\right)
	\leq
	2\cdot
	\frac{\sqrt{d}}{(1-\gamma)^{\frac{3}{2}}}
	\cdot
	\frac{2C^{\frac{1}{2p}}}{1-\gamma}
	\cdot
	\left(\frac{4T}{n}\log(|\+F|T/\delta)\right)^{\frac{1}{2p}}
	\leq
	\w~O\left(
	\frac
	{\sqrt{d}\big(C\log(|\+F|T/\delta)\big)^{\frac{1}{2p}}}
	{(1-\gamma)^{\frac{5}{2}}\cdot n^{\frac{1}{2p}}}
	\right).
\end{equation*}
\end{proof}

\subsection{Proof of Lemma~\ref{lem:tabular}}\label{app:tabular}
\begin{proof}
The bracketing number of the probability simplex $\Delta(|\+X||\+A|)$ is bounded by $N_{[]}(\epsilon,\Delta(|\+X||\+A|),\|\cdot\|_\infty)\leq (c/\epsilon)^{|\+X||\+A|}$ where $c$ is a constant. Hence, we have $N_{[]}(\epsilon,(\Delta(|\+X||\+A|))^{|\+X||\+A|},\|\cdot\|_\infty)\leq (c/\epsilon)^{|\+X|^2|\+A|^2}$.

Let $\w~\Delta$ denote an $\epsilon$-bracket of $(\Delta(|\+X||\+A|))^{|\+X||\+A|}$. Then we can construct a bracket of $\+F_h$ as follows
	\begin{equation*}
		\w~{\+F}_h=\left\{
		\Big[\underline{f},\overline{f}\Big]
		:
		\underline{f}(x,a)=\sum_{x',a'}\underline{w}_{x,a}(x',a')r_H(x',a'),\ 
		\overline{f}(x,a)=\sum_{x',a'}\overline{w}_{x,a}(x',a')r_H(x',a'),\ 
		\forall [\underline{w},\overline{w}]\in\w~\Delta
		\right\}.
	\end{equation*}
	We claim that $\w~{\+F}_h$ is a $\epsilon r_\infty|\+X||\+A|$-bracket of $\+F_h$. To see this, we have
	\begin{align*}
		\Big\|\overline{f}-\underline{f}\Big\|_\infty
		\leq
		\sum_{x',a'}\left|\underline{w}_{x,a}(x',a')-\overline{w}_{x,a}(x',a')\right|r_H(x',a')
		\leq \epsilon\sum_{x',a'}r_H(x',a')
		\leq \epsilon r_\infty|\+X||\+A|.
	\end{align*}
	Therefore, we conclude $N_{[]}(\epsilon r_\infty|\+X||\+A|, \+F_h, \|\cdot\|_\infty)\leq|\w~\Delta|\leq (c/\epsilon)^{|\+X|^2|\+A|^2}$. By substitution we arrive at $N_{[]}(\epsilon , \+F_h, \|\cdot\|_\infty)\leq(cr_\infty|\+X||\+A|/\epsilon)^{|\+X|^2|\+A|^2}$. Then we complete the proof by taking a logatithm.
\end{proof}

\subsection{Proof of Lemma~\ref{lem:lqr-closed-form}}\label{app:pf-lqr}
\begin{align*}
	\mu_h(x,a)=
	&\sum_{i=h}^H-(x_i^\top Qx_i+a_i^\top Ra_i)
	=
	-x^\top Qx-a^\top Ra-\sum_{i=h+1}^H-(x_i^\top Qx_i+a_i^\top Ra_i)\\
	=&
	-x^\top Qx-a^\top Ra-\sum_{i=h+1}^H(x_i^\top Qx_i+x_i^\top K^\top R K x_i)\\
	=&
	-x^\top Qx-a^\top Ra-\sum_{i=h+1}^Hx_i^\top \left( Q + K^\top R K\right) x_i\\
	=&
	-x^\top Qx-a^\top Ra-\sum_{i=h+1}^H \left((A+BK)^{i-h-1}(Ax+Ba)\right)^\top \left(Q + K^\top R K\right) \left((A+BK)^{i-h-1} (Ax+Ba)\right)\\
	=&
	-x^\top Qx-a^\top Ra- (Ax+Ba)^\top\left(\sum_{i=h+1}^H\left((A+BK)^{i-h-1}\right)^\top \left(Q + K^\top R K\right) (A+BK)^{i-h-1}\right) (Ax+Ba).
\end{align*}

\subsection{Proof of Lemma~\ref{lem:lqr-complextiy}}\label{app:pf-lqr-complexity}

\begin{lemma}\label{lem:exp-diff}
	For any $x,a,b\in\=R$, we have $\exp(-(x-a)^2)-\exp(-(x-b)^2)\leq\sqrt{2/e}\cdot|a-b|$.
\end{lemma}
\begin{proof}[Proof of Lemma~\ref{lem:exp-diff}]
	When $a\geq b$, it is equivalent to $\exp(-(x-a)^2)-\exp(-(x-b)^2)\leq \sqrt{2/e}\cdot(a-b)$. Thus it suffices to show that $g(x,a)\coloneqq\exp(-(x-a)^2)-\sqrt{2/e}\cdot a$ is non-increasing in $a$. We take the first derivative with respect to $a$ and then get
	\begin{equation*}
		\frac{\partial }{\partial a}g(x,a)
		=2(x-a)\exp\big(-(x-a)^2\big)-\sqrt{\frac{2}{e}}\le0
	\end{equation*}
	since it is easy to verify that $\max_{x}|x\exp(-x^2)|\leq1/\sqrt{2e}$. This completes the proof for $a\geq b$.
	
	When $a< b$, it suffices to show that $h(x,a)\coloneqq\exp(-(x-a)^2)+\sqrt{2/e}\cdot a$ is non-decreasing in $a$. We take the first derivative with respect to $a$ and then get
	\begin{equation*}
		\frac{\partial }{\partial a}h(x,a)
		=2(x-a)\exp\big(-(x-a)^2\big)+\sqrt{\frac{2}{e}}\ge0.
	\end{equation*}
	Thus we are done.
\end{proof}

\begin{lemma}\label{lem:diff-gaussian}
For any $\mu_1,\mu_2\in\=R$, it holds that $\max_{x}\+N(x\,|\,\mu_1,\sigma^2)-\+N(x\,|\,\mu_2,\sigma^2) \leq \frac{1}{\sigma^2\sqrt{2\pi e}}\cdot|\mu_1-\mu_2|$.
\end{lemma}
\begin{proof}[Proof of Lemma~\ref{lem:diff-gaussian}]
\begin{align*}
	&\+N(x\,|\,\mu_1,\sigma^2)-\+N(x\,|\,\mu_2,\sigma^2)=
	\frac{1}{\sigma\sqrt{2\pi}}\left(\exp\left(-\frac{1}{2}\left(\frac{x-\mu_1}{\sigma}\right)^2\right)-\exp\left(-\frac{1}{2}\left(\frac{x-\mu_2}{\sigma}\right)^2\right)\right)\\
	\leq&\frac{1}{\sigma\sqrt{2\pi}}\cdot\sqrt{\frac{2}{e}}\cdot\left|\frac{\mu_1}{\sigma\sqrt{2}}-\frac{\mu_2}{\sigma\sqrt{2}}\right|
	=\frac{1}{\sigma^2\sqrt{2\pi e}}|\mu_1-\mu_2|
\end{align*}
where the inequality holds for Lemma~\ref{lem:exp-diff}.
\end{proof}

\begin{lemma}\label{lem:lqr-complextiy-cover}
	For LQR, let $\+M_i$ ($i=1,2,3$) denotes the set of possible matrices of $M_i$. We assume that, there exists parameters $m_x$ and $m_a$ for which $\|x\|_2\leq m_x$ and $\|a\|_2\leq m_a$ for all $x\in\+X$ and $a\in\+A$. Then we have
	\begin{align*}
		N_{[]}&(\epsilon, \+F_h, \|\cdot\|_\infty)\le \prod_{i=1,2,3} N\left(\frac{\epsilon\sigma^2(H-h+1)\sqrt{2\pi e}}{2(m_x^2+m_xm_a+m_a^2)},\+M_i,\|\cdot\|_{\rm F}\right).
	\end{align*}
	Here $N_{[]}()$ and $N()$ denote the bracketing number and covering number, respectively.
\end{lemma}

\begin{proof}[Proof of Lemma~\ref{lem:lqr-complextiy-cover}]

We denote by $\w~{\+M}_1$, $\w~{\+M}_2$, and $\w~{\+M}_3$ the $\epsilon$-covers of $\+M_1$, $\+M_2$, and $\+M_3$, respectively. We construct the following function class
\begin{align*}
	\w~{\+F}_h=\Big\{
	\~f(\cdot|x,a)=
	\+N\big(\cdot\,\big|\,
	x^\top \w~M_1 x
	+
	a^\top \w~M_2 x
	+
	a^\top \w~M_3 a
	,
	(H-h+1)\sigma^2
	\big),\ 
	\forall \w~M_1\in\w~{\+M}_1,\w~M_2\in\w~{\+M}_2,\w~M_3\in\w~{\+M}_3
	\Big\}.
\end{align*}
We claim that $\w~{\+F}_h$ is a cover of $\+F_h$. To see this, note that for any $f\in\+F_h$, there exists $\~f\in\w~{\+F}_h$ ($i=1,2,3$) for which $\|M_i-\w~{M}_i\|_{\rm F}\le \epsilon$, and thus
\begin{align*}
	\left\|\~f-f\right\|_\infty
	=&
	\max_{x,a,z}\Big|\+N\big(z\,\big|\,
	x^\top M_1 x
	+
	a^\top M_2 x
	+
	a^\top M_3 a
	,
	(H-h+1)\sigma^2
	\big)\\
	&-
	\+N\big(z\,\big|\,
	x^\top \w~M_1 x
	+
	a^\top \w~M_2 x
	+
	a^\top \w~M_3 a
	,
	(H-h+1)\sigma^2
	\big)\Big|\\
	\le&
	\frac{1}{(H-h+1)\sigma^2\sqrt{2\pi e}}
	\underbrace{
	\left|
	x^\top (M_1-\w~M_1) x
	+
	a^\top (M_2-\w~M_2) x
	+
	a^\top (M_3-\w~M_3) a
	\right|}_{(\heartsuit)}.
\end{align*}
where the last inequality holds for Lemma~\ref{lem:diff-gaussian}. For $\rm(\heartsuit)$, we have
\begin{align*}
	(\heartsuit)
	\leq
	\|x\|_2\|M_1-\w~M_1\|_{\rm F}\|x\|_2
	+
	\|a\|_2\|M_2-\w~M_2\|_{\rm F}\|x\|_2
	+
	\|a\|_2\|M_3-\w~M_3\|_{\rm F}\|a\|_2.
	\leq\epsilon(m_x^2+m_xm_a+m_a^2).
\end{align*}
Hence, we have
\begin{equation*}
	\left\|\~f-f\right\|_\infty\leq \epsilon\cdot\frac{m_x^2+m_xm_a+m_a^2}{(H-h+1)\sigma^2\sqrt{2\pi e}}.
\end{equation*}
This implies 
\begin{equation*}
	N\left(\frac{\epsilon(m_x^2+m_xm_a+m_a^2)}{(H-h+1)\sigma^2\sqrt{2\pi e}},\+F_h,\|\cdot\|_\infty\right)\leq
	N(\epsilon,\+M_1,\|\cdot\|_{\rm F})\cdot N(\epsilon,\+M_2,\|\cdot\|_{\rm F})\cdot N(\epsilon,\+M_3,\|\cdot\|_{\rm F}).
\end{equation*}
We note that $N_{[]}(2\epsilon,\+F_h,\|\cdot\|_\infty)\leq N(\epsilon,\+F_h,\|\cdot\|_\infty)$.
Hence we complete the proof.
\end{proof}

\begin{proof}[Proof of Lemma~\ref{lem:lqr-complextiy}]
Let $\+M_i=\{M:\|M\|_{\rm F}\le m_i\}$ ($i=1,2,3$) denote the set of possible matrices $M_i$. Then we have $N(\epsilon,\+M_1,\|\cdot\|_{\rm F})\leq(3m_1/\epsilon)^{d_x\times d_x}$, $N(\epsilon,\+M_2,\|\cdot\|_{\rm F})\leq(3m_2/\epsilon)^{d_x\times d_a}$, and $N(\epsilon,\+M_3,\|\cdot\|_{\rm F})\leq(3m_3/\epsilon)^{d_a\times d_a}$. By Lemma~\ref{lem:lqr-complextiy-cover}, we have that  
	\begin{align*}
	& N_{[]}(\epsilon, \+F_h, \|\cdot\|_\infty)
	\\
 \leq&
		\left(\frac{6m_1(m_x^2+m_xm_a+m_a^2)}{\epsilon\sigma^2(H-h+1)\sqrt{2\pi e}}\right)^{d_x\times d_x}
		\left(\frac{6m_2(m_x^2+m_xm_a+m_a^2)}{\epsilon\sigma^2(H-h+1)\sqrt{2\pi e}}\right)^{d_x\times d_a}
		\left(\frac{6m_3(m_x^2+m_xm_a+m_a^2)}{\epsilon\sigma^2(H-h+1)\sqrt{2\pi e}}\right)^{d_a\times d_a}\\
	\leq&
		\left(\frac{6m_1(m_x^2+m_xm_a+m_a^2)}{\epsilon\sigma^2\sqrt{2\pi e}}\right)^{d_x\times d_x}
		\left(\frac{6m_2(m_x^2+m_xm_a+m_a^2)}{\epsilon\sigma^2\sqrt{2\pi e}}\right)^{d_x\times d_a}
		\left(\frac{6m_3(m_x^2+m_xm_a+m_a^2)}{\epsilon\sigma^2\sqrt{2\pi e}}\right)^{d_a\times d_a}.
	\end{align*}
	Taking a logarithm on both sides, we get
	\begin{align*}
		\log N_{[]}(&\epsilon, \+F_h, \|\cdot\|_\infty)\\
		\le
		&O
		\left(
		d_x^2\log \frac{m_1(m_x^2+m_xm_a+m_a^2)}{\epsilon\sigma^2}
		+
		d_xd_a\log \frac{m_2(m_x^2+m_xm_a+m_a^2)}{\epsilon\sigma^2}
		+
		d_a^2\log \frac{m_3(m_x^2+m_xm_a+m_a^2)}{\epsilon\sigma^2}
		\right).
	\end{align*}
\end{proof}

\section{Experiment Details}\label{app:exp-details}

We release our code at \url{https://github.com/ziqian2000/Fitted-Likelihood-Estimation}.

\subsection{Implementation Details of Combination Lock Environment}

We first clarify our implementation of the combination lock environment.  

\paragraph{Reward.} 
We denote $r^+$ and $r^-$ as the random reward for latent state $w_H=0$ and $w_H=1$, respectively. For the one-dimensional case, they are sampled from Gaussian distributions: $r^+\sim\+N(1,0.1^2)$ and $r^-\sim\+N(-1,0.1^2)$. For the second experiment with two-dimensional reward, they are defined as
\begin{align*}
	r^+= x+\frac{2x}{\|x\|_2}
	\quad\text{where}\quad
	x\sim\+N\left(\begin{bmatrix}
		0\\0
	\end{bmatrix},\ 
	\begin{bmatrix}
		0.05 & 0\\
		0 & 0.05
	\end{bmatrix}\right)
	,\quad
	r^-\sim\+N\left(\begin{bmatrix}
		0\\0
	\end{bmatrix},\ 
	\begin{bmatrix}
		0.05 & 0\\
		0 & 0.05
	\end{bmatrix}\right).
\end{align*}
Visually, most samples of $r^+$ appear in a ring centered at the origin with a radius of 2.

\paragraph{State.}
The state is constructed by three components, that is, state $x=(x_1,x_2,x_3)^\top$ for which $x_1$ is the one-hot encoding of latent state, $x_2$ is the one-hot encoding of time step $h$, and $x_3$ is a vector of Gaussian noise sampled independently from $\+N(0,0.1^2)$.

\ \\
The optimal action $a^\star_h$ is chosen to be 0 for all $h\in[H]$ for simplicity. We list other environment hyperparameters in \cref{tab:hyperpara} for reference.
\begin{table}[htb]
\caption{Hyperparameters for the combination lock environment. The two columns denote the respective hyperparameters employed in one-dimensional and two-dimensional experiments.}
\label{tab:hyperpara}
\vskip 0.15in
\begin{center}
\begin{small}
\begin{sc}
\begin{tabular}{ccccc}
\toprule
& 1-dimensional & 2-dimensional\\
\midrule
Horizon					&	20	&	10	\\
Number of Actions		&	2	&	2	\\
Dimension of States		&	30	&	30	\\
\bottomrule
\end{tabular}
\end{sc}
\end{small}
\end{center}
\vskip -0.1in
\end{table}

\subsection{Implementation Details of Algorithms}\label{app:impl-alg}

All algorithms, with the exception of Diff-FLE, is implemented by a neural network consisting of two layers, each with 32 neurons, connected by the ReLU activation functions. Diff-FLE employs a three-layered neural network, each layer containing 256 neurons, connected by the ReLU functions. 
Some shared hyperparameters are listed in Table~\ref{tab:hyperpara-common}.
\begin{table}[htb]
\caption{Shared hyperparameters. Note that the size of the dataset is written as a product, which is determined by the way we generate the offline data: the first number means the number of samples generated for each latent state and each time step, the second number means the number of time steps (i.e., horizon), and the third number means the size of the latent space.}
\label{tab:hyperpara-common}
\vskip 0.15in
\begin{center}
\begin{small}
\begin{sc}
\begin{tabular}{ccccc}
\toprule
& 1-dimensional & 2-dimensional\\
\midrule
Size of Dataset	&	$10000\times20\times2$	&	$10000\times10\times2$	\\
Batch Size		&	500		&	500		\\
\bottomrule
\end{tabular}
\end{sc}
\end{small}
\end{center}
\vskip -0.1in
\end{table}

\paragraph{Categorical Algorithm.}
We present the implementation of the two-dimensional version of the categorical algorithm, which is not presented in the prior work~\citep{bellemare2017distributional}. As a reminder, for the one-dimensional counterpart, for each atom of the next state, we first calculate its target position, then distribute the probability of that atom based on the distance of the target position to the closest two atoms. In the two-dimensional case, we discretize on each dimension, resulting in a grid-shaped discretization. Therefore, the probability of the atoms of the next state will be distributed based on the distance to the \emph{four} closest atoms (generally, it will be distributed to $2^n$ atoms in the $n$-dimensional case). The other implementation details are the same as the one-dimensional case. The list of hyperparameters can be found in the Table~\ref{tab:hyperpara-cate}.

\begin{table}[htb]
\caption{Hyperparameters for the categorical algorithm.}
\label{tab:hyperpara-cate}
\vskip 0.15in
\begin{center}
\begin{small}
\begin{sc}
\begin{tabular}{ccccc}
\toprule
& 1-dimensional & 2-dimensional\\
\midrule
Number of Atoms	&	$100$	&	$30^2$	\\
Learning Rate	&$10^{-2}$ 	&$3\times10^{-2}$ 	\\
Number of Iterations		&	200	&	100	\\
Discretized Range	&	$[-1.5, 1.5]$	&	$[-4,4]^2$\\
\bottomrule
\end{tabular}
\end{sc}
\end{small}
\end{center}
\vskip -0.1in
\end{table}

\paragraph{Quantile Algorithm.}
We followed the implementation of \citet{dabney2018distributional}. The list of hyperparameters can be found in the Table~\ref{tab:hyperpara-quan}.
\begin{table}[htb]
\caption{Hyperparameters for quantile Algorithm.}
\label{tab:hyperpara-quan}
\vskip 0.15in
\begin{center}
\begin{small}
\begin{sc}
\begin{tabular}{ccccc}
\toprule
& 1-dimensional\\
\midrule
Number of Quantiles	&	100	\\
Learning Rate	&$10^{-3}$ 	\\
Number of Iterations		&	1000\\
\bottomrule
\end{tabular}
\end{sc}
\end{small}
\end{center}
\vskip -0.1in
\end{table}

\paragraph{Diff-FLE.}
Our implementation is based on DDPM~\citep{ho2020denoising}. However, our neural network is much simpler than theirs, as mentioned above. The list of hyperparameters can be found in the Table~\ref{tab:hyperpara-diff}.
\begin{table}[htbp!]
\caption{Hyperparameters for Diff-FLE.}
\label{tab:hyperpara-diff}
\vskip 0.15in
\begin{center}
\begin{small}
\begin{sc}
\begin{tabular}{ccccc}
\toprule
& 1-dimensional & 2-dimensional\\
\midrule
Steps of Diffusion Process	&	200	&	200	\\
Staring Variance		&	$10^{-3}$	&	$10^{-3}$\\
Final Variance		&	0.1	&	0.1	\\
Variance Increasing	& Linear	& Linear\\
Learning Rate	&$10^{-3}$ 	&$10^{-3}$ 	\\
Number of Iterations		&	5000	&	15000	\\
\bottomrule
\end{tabular}
\end{sc}
\end{small}
\end{center}
\end{table}

\paragraph{GMM-FLE.}
For the training of GMM-FLE, we applied gradient ascent on the log-likelihood. While many classic approaches (e.g., the Expectation-Maximization (EM) algorithm) exist, we found no significant performance gap between gradient ascent and EM in our trials on both one-dimensional and two-dimensional data. Therefore, we opted for the gradient ascent, which matches our theory better. The list of hyperparameters is listed in Table~\ref{tab:hyperpara-gmm}.
\begin{table}[htbp!]
\caption{Hyperparameters for GMM-FLE.}
\label{tab:hyperpara-gmm}
\vskip 0.15in
\begin{center}
\begin{small}
\begin{sc}
\begin{tabular}{ccccc}
\toprule
& 1-dimensional & 2-dimensional\\
\midrule
Number of Gaussian Distribution	&	10	&	10	\\
Learning Rate		&	$10^{-4}$	&	$2\times10^{-4}$	\\
Number of Iterations		&	20000	&	10000	\\
\bottomrule
\end{tabular}
\end{sc}
\end{small}
\end{center}
\vskip -0.1in
\end{table}

\subsection{Full Experiment Results}\label{app:full-exp-results}

\cref{tab:1d-full} is the full version of Table~\ref{tab:1d}. Table~\ref{tab:2d-full} is the full version of \cref{tab:2d}. 

Table~\ref{tab:1d-wass-full} is the counterpart of Table~\ref{tab:1d} but using $1$-Wasserstein distance. It is computed in a similar way as Table~\ref{tab:1d}: we first sample 20k values from each distribution and then compute the 1-Wasserstein distance between the empirical distributions. 
All configurations are the same as that for the total variation distance experiment, except that we set the learning rate of the quantile algorithm to $10^{-1}$ for smaller 1-Wasserstein error.
We can see that GMM-FLE achieves the smallest Wasserstein distance in most steps (except $h=1$ and $2$). This result aligns with what we observed in Table~\ref{tab:1d}.

\begin{table}[htbp!]
\caption{Full version of Table~\ref{tab:1d}.}
\label{tab:1d-full}
\vskip 0.15in
\begin{center}
\begin{small}
\begin{sc}
\begin{tabular}{ccccc}
\toprule
$h$ & Cate Alg & Quan Alg & Diff-FLE & GMM-FLE \\
\midrule
1	&	0.071 	$\pm$	0.015 	&		0.603 	$\pm$	0.011 	&		0.292 	$\pm$	0.073 	&		\textbf{0.039 	$\pm$	0.004} 	\\
2	&	0.067 	$\pm$	0.012 	&		0.609 	$\pm$	0.014 	&		0.305 	$\pm$	0.055 	&		\textbf{0.041 	$\pm$	0.005} 	\\
3	&	0.068 	$\pm$	0.013 	&		0.612 	$\pm$	0.017 	&		0.305 	$\pm$	0.079 	&		\textbf{0.039 	$\pm$	0.009} 	\\
4	&	0.073 	$\pm$	0.013 	&		0.593 	$\pm$	0.015 	&		0.288 	$\pm$	0.073 	&		\textbf{0.038 	$\pm$	0.003} 	\\
5	&	0.074 	$\pm$	0.015 	&		0.602 	$\pm$	0.009 	&		0.285 	$\pm$	0.054 	&		\textbf{0.036 	$\pm$	0.009} 	\\
6	&	0.077 	$\pm$	0.011 	&		0.612 	$\pm$	0.010 	&		0.268 	$\pm$	0.040 	&		\textbf{0.030 	$\pm$	0.008} 	\\
7	&	0.080 	$\pm$	0.014 	&		0.602 	$\pm$	0.014 	&		0.290 	$\pm$	0.066 	&		\textbf{0.034 	$\pm$	0.004} 	\\
8	&	0.080 	$\pm$	0.016 	&		0.584 	$\pm$	0.018 	&		0.273 	$\pm$	0.039 	&		\textbf{0.039 	$\pm$	0.013} 	\\
9	&	0.081 	$\pm$	0.019 	&		0.529 	$\pm$	0.028 	&		0.247 	$\pm$	0.034 	&		\textbf{0.048 	$\pm$	0.010} 	\\
10	&	0.079 	$\pm$	0.017 	&		0.494 	$\pm$	0.018 	&		0.234 	$\pm$	0.043 	&		\textbf{0.044 	$\pm$	0.012} 	\\
11	&	0.080 	$\pm$	0.016 	&		0.514 	$\pm$	0.018 	&		0.244 	$\pm$	0.038 	&		\textbf{0.039 	$\pm$	0.012} 	\\
12	&	0.089 	$\pm$	0.009 	&		0.518 	$\pm$	0.013 	&		0.232 	$\pm$	0.015 	&		\textbf{0.032 	$\pm$	0.007} 	\\
13	&	0.089 	$\pm$	0.011 	&		0.481 	$\pm$	0.016 	&		0.219 	$\pm$	0.027 	&		\textbf{0.029 	$\pm$	0.016} 	\\
14	&	0.081 	$\pm$	0.015 	&		0.416 	$\pm$	0.026 	&		0.221 	$\pm$	0.021 	&		\textbf{0.033 	$\pm$	0.012} 	\\
15	&	0.083 	$\pm$	0.015 	&		0.330 	$\pm$	0.028 	&		0.178 	$\pm$	0.033 	&		\textbf{0.026 	$\pm$	0.015} 	\\
16	&	0.081 	$\pm$	0.009 	&		0.283 	$\pm$	0.017 	&		0.170 	$\pm$	0.045 	&		\textbf{0.027 	$\pm$	0.013} 	\\
17	&	0.082 	$\pm$	0.008 	&		0.252 	$\pm$	0.008 	&		0.167 	$\pm$	0.037 	&		\textbf{0.034 	$\pm$	0.013} 	\\
18	&	0.070 	$\pm$	0.010 	&		0.217 	$\pm$	0.012 	&		0.133 	$\pm$	0.019 	&		\textbf{0.023 	$\pm$	0.008} 	\\
19	&	0.078 	$\pm$	0.011 	&		0.167 	$\pm$	0.019 	&		0.109 	$\pm$	0.031 	&		\textbf{0.018 	$\pm$	0.008} 	\\
20	&	0.077 	$\pm$	0.014 	&		0.076 	$\pm$	0.009 	&		0.067 	$\pm$	0.024 	&		\textbf{0.013 	$\pm$	0.005} 	\\
\bottomrule
\end{tabular}
\end{sc}
\end{small}
\end{center}
\vskip -0.1in
\end{table}

\begin{table}[htbp!]
\caption{Full version of Table~\ref{tab:2d}.}
\label{tab:2d-full}
\vskip 0.15in
\begin{center}
\begin{small}
\begin{sc}
\begin{tabular}{ccccc}
\toprule
$h$ & Cate Alg & Diff-FLE & GMM-FLE \\
\midrule
1	&	0.483 	$\pm$	0.003 		&	\textbf{0.357 	$\pm$	0.031} 		&	0.438 	$\pm$	0.008 	\\
2	&	0.483 	$\pm$	0.003 		&	\textbf{0.344 	$\pm$	0.030} 		&	0.424 	$\pm$	0.048 	\\
3	&	0.480 	$\pm$	0.003 		&	\textbf{0.339 	$\pm$	0.023} 		&	0.450 	$\pm$	0.042 	\\
4	&	0.469 	$\pm$	0.002 		&	\textbf{0.327 	$\pm$	0.019} 		&	0.478 	$\pm$	0.048 	\\
5	&	0.466 	$\pm$	0.001 		&	\textbf{0.310 	$\pm$	0.019} 		&	0.493 	$\pm$	0.050 	\\
6	&	0.466 	$\pm$	0.001 		&	\textbf{0.289 	$\pm$	0.031} 		&	0.491 	$\pm$	0.061 	\\
7	&	0.470 	$\pm$	0.003 		&	\textbf{0.256 	$\pm$	0.032} 		&	0.510 	$\pm$	0.080 	\\
8	&	0.465 	$\pm$	0.002 		&	\textbf{0.234 	$\pm$	0.023} 		&	0.505 	$\pm$	0.099 	\\
9	&	0.453 	$\pm$	0.001 		&	\textbf{0.207 	$\pm$	0.014} 		&	0.502 	$\pm$	0.094 	\\
10	&	0.446 	$\pm$	0.002 		&	\textbf{0.143 	$\pm$	0.011} 		&	0.376 	$\pm$	0.101 	\\
\bottomrule
\end{tabular}
\end{sc}
\end{small}
\end{center}
\vskip -0.1in
\end{table}

\begin{table}[htbp!]
\caption{Approximated $d_{w,1}$ between $\E_{x\sim\psi(0,h)}\^f_h(x,a^\star_h)$ and $\E_{x\sim\psi(0,h)}Z^\pi_h(x,a^\star_h)$ in the 1-d case. The means and standard errors are computed via five independent runs.}
\label{tab:1d-wass-full}
\vskip 0.15in
\begin{center}
\begin{small}
\begin{sc}
\begin{tabular}{ccccc}
\toprule
$h$ & Cate Alg & Quan Alg & Diff-FLE & GMM-FLE \\
\midrule
1	&	\textbf{0.056	$\pm$	0.047}	&	0.144	$\pm$	0.015	&	0.150	$\pm$	0.060	&	0.062	$\pm$	0.009	\\
2	&	\textbf{0.053	$\pm$	0.045}	&	0.141	$\pm$	0.014	&	0.153	$\pm$	0.040	&	0.060	$\pm$	0.008	\\
3	&	0.065	$\pm$	0.049	&	0.136	$\pm$	0.014	&	0.127	$\pm$	0.056	&	\textbf{0.049	$\pm$	0.010}	\\
4	&	0.072	$\pm$	0.052	&	0.133	$\pm$	0.014	&	0.148	$\pm$	0.068	&	\textbf{0.063	$\pm$	0.007}	\\
5	&	0.074	$\pm$	0.045	&	0.127	$\pm$	0.010	&	0.136	$\pm$	0.061	&	\textbf{0.040	$\pm$	0.015}	\\
6	&	0.079	$\pm$	0.050	&	0.125	$\pm$	0.014	&	0.107	$\pm$	0.041	&	\textbf{0.031	$\pm$	0.017}	\\
7	&	0.087	$\pm$	0.053	&	0.122	$\pm$	0.016	&	0.127	$\pm$	0.045	&	\textbf{0.036	$\pm$	0.019}	\\
8	&	0.090	$\pm$	0.059	&	0.120	$\pm$	0.020	&	0.108	$\pm$	0.038	&	\textbf{0.051	$\pm$	0.027}	\\
9	&	0.092	$\pm$	0.062	&	0.109	$\pm$	0.022	&	0.138	$\pm$	0.062	&	\textbf{0.054	$\pm$	0.037}	\\
10	&	0.082	$\pm$	0.044	&	0.110	$\pm$	0.020	&	0.122	$\pm$	0.085	&	\textbf{0.039	$\pm$	0.027}	\\
11	&	0.090	$\pm$	0.051	&	0.105	$\pm$	0.027	&	0.145	$\pm$	0.095	&	\textbf{0.030	$\pm$	0.014}	\\
12	&	0.090	$\pm$	0.050	&	0.100	$\pm$	0.022	&	0.109	$\pm$	0.071	&	\textbf{0.022	$\pm$	0.017}	\\
13	&	0.091	$\pm$	0.043	&	0.088	$\pm$	0.025	&	0.140	$\pm$	0.059	&	\textbf{0.024	$\pm$	0.023}	\\
14	&	0.066	$\pm$	0.054	&	0.089	$\pm$	0.019	&	0.110	$\pm$	0.029	&	\textbf{0.026	$\pm$	0.011}	\\
15	&	0.067	$\pm$	0.042	&	0.073	$\pm$	0.015	&	0.104	$\pm$	0.045	&	\textbf{0.020	$\pm$	0.018}	\\
16	&	0.070	$\pm$	0.051	&	0.075	$\pm$	0.015	&	0.114	$\pm$	0.080	&	\textbf{0.021	$\pm$	0.018}	\\
17	&	0.047	$\pm$	0.028	&	0.060	$\pm$	0.013	&	0.077	$\pm$	0.017	&	\textbf{0.023	$\pm$	0.017}	\\
18	&	0.026	$\pm$	0.015	&	0.051	$\pm$	0.008	&	0.053	$\pm$	0.015	&	\textbf{0.012	$\pm$	0.009}	\\
19	&	0.041	$\pm$	0.012	&	0.043	$\pm$	0.010	&	0.048	$\pm$	0.023	&	\textbf{0.009	$\pm$	0.004}	\\
20	&	0.023	$\pm$	0.004	&	0.020	$\pm$	0.005	&	0.017	$\pm$	0.008	&	\textbf{0.004	$\pm$	0.001}	\\
\bottomrule
\end{tabular}
\end{sc}
\end{small}
\end{center}
\vskip -0.1in
\end{table}

\end{document}